\documentclass{article}

\usepackage{microtype}
\usepackage{graphicx}
\usepackage{subfigure}
\usepackage{booktabs} 

\usepackage{hyperref}

\usepackage[accepted]{icml2023}

\usepackage{amsmath}
\usepackage{amssymb}
\usepackage{mathtools}
\usepackage{amsthm}

\usepackage[capitalize,noabbrev]{cleveref}

\theoremstyle{plain}
\newtheorem{theorem}{Theorem}[section]
\newtheorem{proposition}[theorem]{Proposition}
\newtheorem{lemma}[theorem]{Lemma}

\theoremstyle{definition}

\theoremstyle{remark}

\usepackage[textsize=tiny]{todonotes}

\usepackage{bbm}
\usepackage{amssymb}
\usepackage{pifont}
\usepackage{xcolor}

\def\ie{{\em i.e.},\ }
\def\eg{{\em e.g.},\ }

\def\Surv#1#2{S(\,#1 \,\mid\, #2\,)}
\def\SurvSub#1#2#3{S_{\text{#1}}(\,#2\mid #3\,)}
\def\SurvTrue#1#2{\SurvSub{true}{#1}{#2} }
\def\SurvKM#1#2{S_{\text{KM}(#2)}(\, #1 \,)}
\def\Bfx#1{\boldsymbol{x}_{#1}}
\def\E{\mathbb{E}}

\def\Data{\mathcal{D}}

\newcommand{\cmark}{\ding{51}}%
\newcommand{\xmark}{\ding{55}}%

\def\Rmsr#1#2#3#4{
 \mathcal{R}_{\text{#1}}(\,#2,\,#3,\,#4\,) }
 
\usepackage[finalnew]{trackchanges}
\addeditor{RG}
\addeditor{RR}
\addeditor{RH}
\addeditor{R2}
\addeditor{R1}
\addeditor{R3}

\icmltitlerunning{An Effective Meaningful Way to Evaluate Survival Models}

\begin{document}

\twocolumn[
\icmltitle{An Effective Meaningful Way to Evaluate Survival Models}




\begin{icmlauthorlist}
\icmlauthor{Shi-ang Qi}{uacs}
\icmlauthor{Neeraj Kumar}{amii}
\icmlauthor{Mahtab Farrokh}{uacs}
\icmlauthor{Weijie Sun}{uacs}
\icmlauthor{Li-Hao Kuan}{uacs} \\
\icmlauthor{Rajesh Ranganath}{nyu}
\icmlauthor{Ricardo Henao}{duke}
\icmlauthor{Russell Greiner}{uacs,amii}
\end{icmlauthorlist}

\icmlaffiliation{uacs}{Computing Science, University of Alberta, Edmonton, Canada}
\icmlaffiliation{amii}{Alberta Machine Intelligence Institute, Edmonton, Canada}
\icmlaffiliation{nyu}{Computer Science \& Center for Data Science, New York University, New York City, USA}
\icmlaffiliation{duke}{Biostatistics \& Bioinformatics, Duke University, Durham, USA}

\icmlcorrespondingauthor{Shi-ang Qi}{shiang@ualberta.ca}
\icmlcorrespondingauthor{Russell Greiner}{rgreiner@ualberta.ca}

\icmlkeywords{Machine Learning, ICML}

\vskip 0.3in
]



\printAffiliationsAndNotice{}  

\begin{abstract}
One straightforward metric to evaluate a survival prediction model is based on the Mean Absolute Error (MAE) -- the average of the absolute difference between the time predicted by the model and the true event time, over all subjects. Unfortunately, this is challenging because, in practice, the test set includes (right) censored individuals, meaning we do not know when a censored individual actually experienced the event. In this paper, we explore various \change[R2]{approaches}{metrics} to estimate MAE for survival datasets that include (many) censored individuals. 
Moreover, we introduce a novel and effective approach for generating realistic semi-synthetic survival datasets to facilitate the evaluation of metrics. 
Our findings, based on the analysis of the semi-synthetic datasets, reveal that our proposed metric (MAE using pseudo-observations) is able to rank models accurately based on their performance, and often closely matches the true MAE -- in particular, is better than several alternative methods.
\end{abstract}

\section{Introduction}
Survival prediction models are often used to predict how long an individual will survive -- or in general, the time until an individual experiences a specific event. These have many applications in medicine (time to death, relapse, or recovery), business (time to service cancellation), and social sciences (war or peace duration).
Unlike typical regression problems, one challenge of training and evaluating a survival prediction model is that survival datasets often contain censored observations~\cite{klein2003survival}.
This paper focuses on 
the most prevalent type of censorship, right-censoring, which provides only a lower bound on event time.
For example, consider a patient who entered a 5-year study at its beginning and was still alive at the end of the study.
We only know this patient survived for at least 5 years, but do not know whether that patient lived a day, a month, or 20 years after the study ended.

Numerous statistical and machine learning models have been developed to estimate survival outcomes from input features. In this paper, we focus on a class of survival models that learns to compute an individual's survival distribution (ISD)~\cite{haider2020effective}: a probability curve for all future time points for a specific patient.
Note that one can use an individual's ISD to 
($i$)~compute that individual's expected time-to-event,
($ii$)~provide single-time estimations (\eg 5-year cancer onset probability, like the Gail model~\cite{costantino1999validation}), or
($iii$)~estimate a risk score 
(like Cox Proportional Hazard model~\cite{cox1972regression}).
Obviously, the computation of aforementioned quantities from ISDs will only be reliable if a model's predicted ISD is  ``accurate''.
One question that plagues the survival prediction community is: 
{\em What is an appropriate scoring rule for evaluating survival models?} 
The answer may vary from task to task -- \eg the concordance index (C-index)~\cite{harrell1996multivariable} is useful in several clinical problems that require comparing patients, such as prioritizing patients for liver transplants (a patient at the highest risk of death should be treated first). 
Figure~\ref{fig:6_metrics} shows a visualization of six typical evaluation metrics (see discussion in Section~\ref{sec:mae_censor}).

\begin{figure*}[ht]
    \centering
    \includegraphics[width=\textwidth]{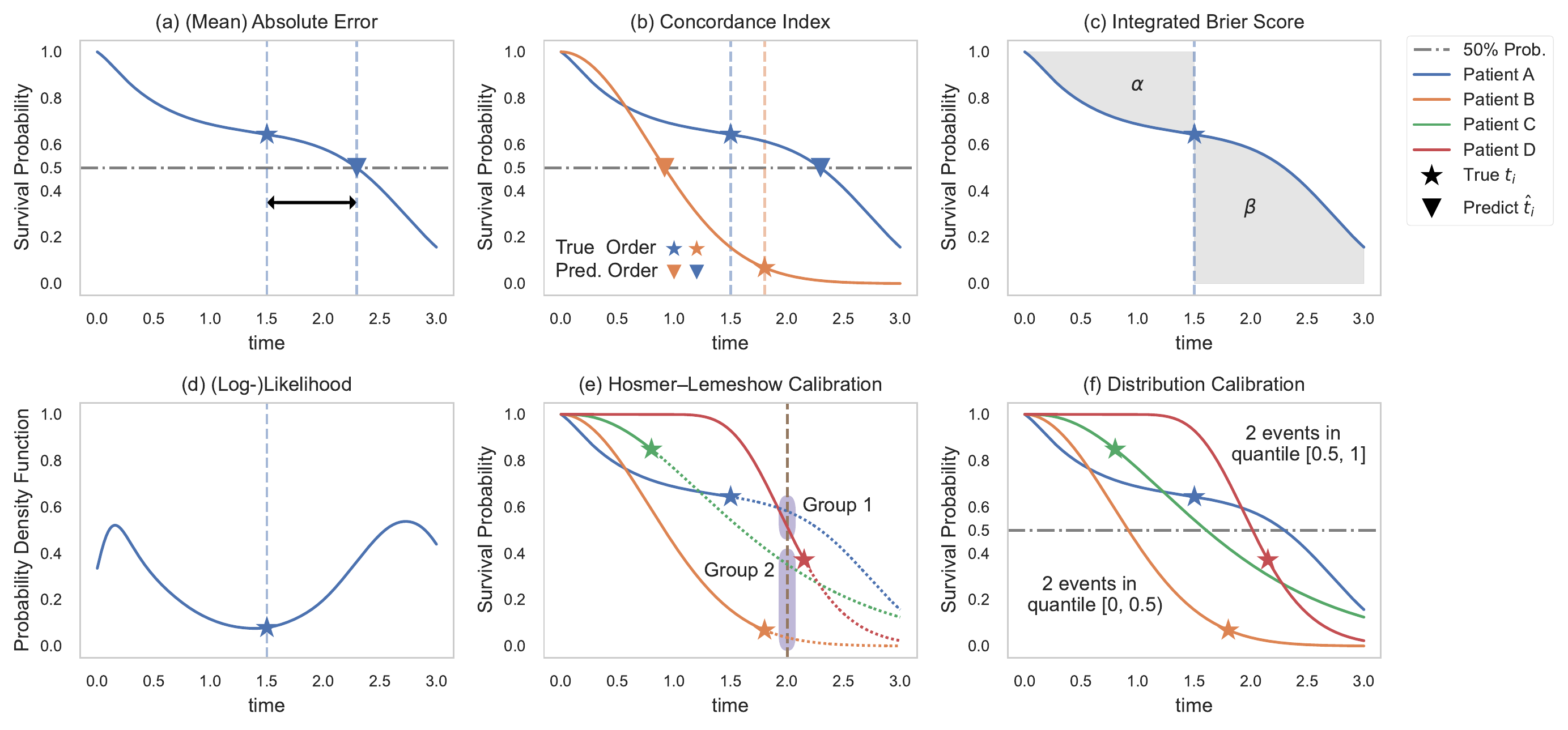}
    \vspace*{-0.3in}
    \caption{Illustration of six common evaluation metrics using four uncensored subjects. Note that the $y$-axis of (d), Log-Likelihood, is the probability density function, while the others are survival functions. Here we use the median survival time of the survival function as the predicted time for a more intuitive visualization. 
    \add[R2]{Please refer to Appendix~\hbox{\ref{sec:six_metrics}} for the detailed discussion of these metrics.}}
    \label{fig:6_metrics}
\end{figure*}

However, the Mean Absolute Error (MAE) seems to be the most intuitive metric for evaluating survival prediction models, as it measures the expected difference between predicted and actual event times.
While this difference is trivial to compute for uncensored individuals,
it is problematic for censored individuals. 
Thus, researchers have proposed several versions of MAE to handle censored individuals, such as MAE-uncensored and MAE-hinge~\cite{haider2020effective}.
\add[]{However}, these approaches often produce biased MAEs for high censoring. 


Here, we propose an MAE-inspired evaluation metric using the pseudo-observation de-censoring technique~\cite{andersen2003generalised}, MAE-PO (defined \change[R2]{below}{in Section~\hbox{\ref{sec:mae-po}}}). 
Our metric effectively deals with a censored subject by ($a$) estimating its pseudo-observation value for the censored subject by calculating how much it counts toward the group-level Kaplan-Meier (KM) estimator~\cite{kaplan1958nonparametric}, and ($b$) employing a weighted scheme to represent the confidence of pseudo-observation estimation.
Our main contributions are:
\begin{enumerate}
\parskip0em
\itemsep0em
    \item We show that MAE is fundamentally different from the other standard evaluation measures (Section~\ref{sec:mae_censor} and Appendix~\ref{sec:six_metrics});
    \item We theoretically prove some relevant properties of pseudo-observation, which helps justify the MAE pseudo-observation approach for evaluating the time-to-event prediction (Section~\ref{sec:mae-po} \add[R2]{and Appendix~\hbox{\ref{appendix:pseudo-obs_property}}});
    \item We provide a way to produce semi-synthetic (but realistic) survival datasets, where we know the true survival time of all instances (Section~\ref{sec:results}), \add[R2]{which is essential for evaluating evaluation methods};
    \item We compare six versions of MAE-inspired metrics (on various realistic semi-synthesized datasets) to determine which metric
    ($i$)~ranks the ISD models in a way that is close to the true MAE, or ($ii$)~reports the error closest to the true MAE (Section~\ref{sec:results});
    \item We show that, in many situations, MAE-PO is the most suitable estimator. We also provide a code base for these MAE approaches, for this and other variants.
\end{enumerate}
We believe 
this is the first methodology for identifying the appropriate evaluation metrics using meaningful semi-synthetic datasets. \remove[]{We will release the semi-synthetic datasets and source code for reproducibility and to support future research, and will also publish a {\tt pip} package to facilitate general survival prediction evaluation for the community.}\note[]{Already add code link in experimental section}

\section{Preliminaries}
\label{sec:notation}
In general, a survival dataset contains \textit{N} time-to-event tuples, $\mathcal{D}=\{(\Bfx{i}, t_i, \delta_i)\}_{i=1}^N$, 
where $\Bfx{i} \in \mathbb{R}^d$ represents the observed $d$-dimensional features for the $i$-th subject,
$t_i \in \mathbb{R}_+$ 
denotes the event or censor time, 
and $\delta_i \in \{0, 1\}$ is a censor/event indicator where $\delta_i = 0$ means the subject is right-censored (subject has not experienced an event at time $t_i$)
and $\delta_i = 1$ means subject died at time $t_i$. 
Conceptually, we assume subject $i$ has an event time $e_i$ and a censoring time $c_i$, and assign 
\change[]{$t_i \leftarrow \min\{e_i, c_i\}$}{$t_i \triangleq \min\{e_i, c_i\}$} and 
\change[]{$\delta_i \leftarrow  \mathbbm{1}[e_i \leq c_i]$}{$\delta_i \triangleq  \mathbbm{1}[e_i \leq c_i]$}. 
We assume \emph{independent censoring}:
$e_i$ and $c_i$ are assumed independent, conditional on the covariates $\Bfx{i}$.

ISD models target the individual survival distribution $\Surv{t}{\Bfx{i}} = P(T > t \mid \mathbf{X} = \Bfx{i})$. 
Further, $F(t \mid \Bfx{i}) = 1 - \Surv{t}{\Bfx{i}}$ denotes the (conditional) cumulative density function and $f(t \mid \Bfx{i}) = \frac{ - \partial \Surv{t}{\Bfx{i}}}{\partial t}$ denotes the (conditional) probability density function (PDF) of the event time $T$. 
A predicted \change[]{failure}{event} time $\hat{t}_i$ can then be represented by either mean (expected) or median survival time, respectively: 
\begin{align}
    \hat{t}_{i, \text{mean}}\ & =\ \E_t[\,\Surv{t}{\Bfx{i}}\,]\ =\ \int_0^\infty \Surv{t}{\Bfx{i}}\, dt \ , \label{eq:mean_survival_time} \\
    \hat{t}_{i, \text{median}}\ & =\ \text{median}(\Surv{t}{\Bfx{i}}) = S^{-1} (\,\tau = 0.5 \mid \Bfx{i} \,) \ , \label{eq:median_survival_time}
\end{align}
%
where $\tau \in [0, 1]$ represents the quantile probability level.
If necessary, a linear extrapolation (extrapolate from the initial to the last time point of the ISD curve) might be applied to ISDs that do not reach 0\% or 50\% quantile probability.

Suppose we have two models, each producing a predicted ISD curve for every individual 
$\{\SurvSub{M1}{t}{\Bfx{i}}\}_{i=1}^N$ and 
$\{\SurvSub{M2}{t}{\Bfx{i}}\}_{i=1}^N$.
Our goal is to find an MAE-inspired metric $\mathcal{R}(\cdot)$ that, given an ISD distribution
$\SurvSub{m}{t}{\Bfx{i}}$, 
the observed time $t_i$ and event indicator $\delta_i$ of a subject, 
returns a good approximation to the true MAE%
\footnote{Here, based on the true event time $e_i$,
which of course is not given for censored instances.
}.
We can then consider the average score, over the dataset, $\mathbb{E}[ \, \mathcal{R}( \cdot) \,]$.
Our goal is a measure whose average is as close as possible to the true MAE.
To be specific, we hope that 
$\mathbb{E}[\, \mathcal{R}(\,\SurvSub{m}{t}{\Bfx{i}},\, t_i, \,\delta_i\,) \, ]$, 
is close to the true MAE score for each model,
or at least, can correctly rank the two models.


\section{Handling Right-Censoring in MAE}
\label{sec:mae_censor}
Survival prediction is like regression as it predicts a real number (the subject's time of the event) from a description of that subject, $\Bfx{i}$. Given this commonality, we want to evaluate survival prediction models using measures for evaluating regression tasks, such as mean absolute error (MAE).
As suggested in Figure~\ref{fig:6_metrics}(a),
the absolute error for an uncensored subject is the absolute difference between the true event time and the predicted time:
\begin{equation}
\label{eq:MAE_uncensor}
    \Rmsr{MAE}{\hat{t}_i}{t_i}{\delta_i = 1} 
    \ =\ |t_i - \hat{t}_i| ,
\end{equation}
where $\hat{t}_i$ is the median
survival time\footnote{Alternatively, we could use the mean survival time for $\hat{t}_i$ from Equation~\ref{eq:mean_survival_time}.} from Equation~\ref{eq:median_survival_time}. 
MAE is formally a {\em negative scoring rule}, 
as more precise models have smaller values.

It is vital to have a thorough grasp of the limitations of all types of evaluation metrics when selecting metrics for model optimization, and separately for model evaluation. 
In this section, we briefly motivate that the MAE score is the most appropriate metric if the objective is to quantify the time-to-event accuracy.
Figure~\ref{fig:6_metrics} shows that, C-index measures the ranking accuracy by assessing if the order of true event times (stars) is concordant with the order of predicted event times (triangles) (b); 
integrated Brier score~\cite{graf1999assessment} measures the accuracy of the predicted probabilities over all times via the weighted squared error of the shaded regions (c);
log-likelihood measures the \add[R2]{magnitude of the} predicted probability at event times \add[R2]{(d);}
Hosmer-Lemeshow calibration~\cite{hosmer1980goodness} \add[R2]{assesses if the expected and observed event rates are statistically similar (e);}
and Distribution calibration (D-calibration)~\cite{haider2020effective} \add[R2]{examines if the proportion of subjects who dies in each quantile interval is uniformly distributed (f)}. 
Appendix~\ref{sec:six_metrics} shows that the time-to-event precision that MAE captures cannot be covered by other metrics.
Moreover, the model preference between MAE and each other metrics can be completely different
-- \ie a model can have perfect C-index but terrible MAE, while another model is good at MAE but has poor C-index score.

However, evaluating (and also learning) survival prediction models is challenging 
when 
the dataset includes 
right-censored subjects,
meaning 
it is critical to define 
learning (and evaluation)
algorithms that can appropriately incorporate the censored subjects.
This section begins by reviewing six MAE variants that claim to handle censored subjects, including three novel MAE variants. 
Section~\ref{sec:results} then presents an empirical comparison of all these versions.

Unless otherwise specified, we will use the median survival time (Equation~\ref{eq:median_survival_time}) as the default method to calculate the predicted time of the ISD curves. 
This is because 
($i$) the combination of MAE and median survival time is a proper scoring rule (Theorem~\ref{them:MAE_median}); and
($ii$) linear extrapolation is only required for curves that do not reach 50\% for median survival time, whereas it is required for every curve that does not reach 0\% for mean survival time -- which is extremely common.

\subsection{MAE-Uncensored}
The simplest solution is to exclude all censored subjects from the evaluation, then use Equation~\ref{eq:MAE_uncensor} to calculate the absolute error for each uncensored patient and take the average over the uncensored instances. 
The (marginal) distribution of censored and event subjects can vary substantially, making this strategy susceptible to bias.
Moreover, when the censoring rate is high, a sizeable portion of the data will be completely ignored by the performance metric.

\subsection{MAE-Hinge}
Another way to incorporate censoring is to use the hinge loss -- a one-sided metric that considers only if the predicted time is earlier than the censored time. 
For a censored subject, the MAE-hinge score is:
\begin{equation*}
\label{eq:MAE_hinge}
\Rmsr{MAE-hinge}{\hat{t}_i}{t_i}{\delta_i = 0}
\ =\ \max \{\,t_i - \hat{t}_i,\, 0\,\} \ .
\end{equation*}
The MAE-hinge is an optimistic evaluation of the true MAE for two reasons: 
($i$)~it assigns a score of 0 if the censoring happens before the predicted survival time; and 
($ii$)~it assigns a loss of $c_i - \hat{t}_i$ if the censoring time occurs after the prediction.
Both are lower or equal to the true prediction error.
Therefore, for a dataset with an extremely high censoring rate, a model can actually obtain an extremely low MAE-hinge by overestimating the event time for all subjects,
as 
MAE-hinge will give zero scores for censored subjects, resulting in an optimistic overall score\footnote{D-calibration can be used in conjunction with MAE-hinge to prevent this type of situation~\cite{qi2022personalized},
as overestimating the event time will result in a skewed proportion for large probability intervals (which means not D-calibrated).}.

\subsection{MAE-Margin}
MAE-margin~\cite{haider2020effective} assigns a ``best guess'' value (margin time) to each censored subject using the non-parametric population KM~\cite{kaplan1958nonparametric} estimator.
This margin time can be interpreted as a conditional expectation of 
the event time given the 
event time is greater than the censoring time. Given a subject censored at time $t_i$, we can calculate its margin time by:
\begin{equation*}
    e_{\text{margin}}(t_i, \Data) = \mathbb{E}_t[e_i \mid e_i > t_i] = t_i + \frac{\int_{t_i}^\infty \SurvKM{t}{\Data} dt}{\SurvKM{t}{\Data}} \ ,
\end{equation*}
where $\SurvKM{t}{\Data}$ is the KM estimation that is typically derived from the training dataset. 

Based on the censored time, the margin time can be more trustworthy for some circumstances than for others. 
For instance, we know effectively nothing about a patient censored at time 0 -- hence we should have very little confidence that the margin time matches its actual event time. 
In contrast, assume that no patients in the training data have ever lived longer than 130 years old.
If a patient was censored at 100 years old (\ie close to the longest known lifespan), we are quite certain that his/her margin time is close to the observed event time.
Therefore, for each censored subject, \citet{haider2020effective} suggested we use a confidence weight $\omega_i = 1 - S_{\text{KM}(\Data)}(t_i)$ for the error calculated based on the margin value. 
This weight $\omega_i$ yields lower confidence for early censoring subjects and higher confidence for late censoring data. 
Of course, we set the weights for uncensored subjects $i$ to $w_i = 1$ as we have full confidence in those error calculations.
The overall MAE-margin after a re-weighting scheme is:
\begin{align}
    & \E_{i \sim \mathcal{D}} [\mathcal{R}_{\text{MAE-margin}} (\hat{t}_i, t_i, \delta_i)] = \label{eq:MAE_margin} \\
    &  \frac{1}{\sum_{i=1}^{N} \omega_i} \sum_{i =1}^{N} \omega_i  \left| [(1 - \delta_i) \cdot e_{\text{margin}}(t_i) + \delta_i \cdot t_i ] - \hat{t}_i \right| \ . \notag
\end{align}

\subsection{MAE-IPCW-D}
\label{sec:ipcw-d}
Inverse Probability Censoring Weight (IPCW) was originally designed for handling censored subjects in the calculation of Brier score (BS)~\cite{graf1999assessment}. The method uniformly transfers a censored subject's weights to subjects with known status at that time~\cite{vock2016adapting}. 
{In its simplest form, IPCW requires completely independent censoring; 
IPCW can also be extended to independent censoring conditional on covariates to even estimate survival models}~\cite{han2021inverse}.

Inspired by 
{the simple form} of IPCW, we can design an MAE-based evaluation method by uniformly transferring the weight of a censored subject to the uncensored subjects with later event times~\cite{graf1999assessment}.
Similar to how the prediction error of a censored subject can be approximated using the deterministic errors of subsequent uncensored subjects in the IPCW 
Brier Score, 
we can formulate 
this new MAE-based evaluation as: 
\begin{equation}
\label{eq:MAE_ipcw_diff}
    \E_{i \sim \mathcal{D}} [\mathcal{R}_{\text{MAE-IPCW-D}} (\hat{t}_i, t_i, \delta_i)] = \frac{1}{N} \sum_{i=1}^N \frac{| t_i - \hat{t}_i | \cdot \mathbbm{1}_{\delta_i = 1}}{G(t_i)} \ ,
\end{equation}
where $G(t_i)$ is the probability of not being censored at the event time. 
We call this method ``MAE-IPCW-D'' (where D stands for difference) since the IPCW reweighing is essentially an approximation of the difference between the expected and observed times.

\subsection{MAE-IPCW-T}
One problem with MAE-IPCW-D is that Equation~\ref{eq:MAE_ipcw_diff} considers only the predicted time for the uncensored subjects, but does not consider the predictions for the censored subjects.
For example, imagine two models, M1 and M2, have identical predictions for every subject, except for one censored subject who is censored at time 100.
M1 predicts this subject will die at 5, but M2 predicts it at 90. 
Notice M1 is off by at least 95 here, and M2 by at least 10; indeed, we know that M1's error for this patient is 85 worse than M2's. However, MAE-IPCW-D gives both models the same score. 

To avoid this, the MAE-IPCW-T (where T stands for time) metric instead produces an estimated surrogate time of the event for each censored subject, as the average over the times of all subsequent uncensored subjects:
\begin{equation}
\label{eq:MAE_ipcw_time}
    e_{\text{IPCW}}(t_i, \Data)\ =\ \frac{\sum_{j = 1}^N \mathbbm{1}_{t_i < t_j} \cdot \mathbbm{1}_{\delta_j = 1} \cdot t_j}{\sum_{j = 1}^N \mathbbm{1}_{t_i < t_j} \cdot \mathbbm{1}_{\delta_j = 1}} \ .
\end{equation}
After calculating this IPCW-weighted time using Equation~\ref{eq:MAE_ipcw_time}, the MAE-IPCW-T method then uses Equation~\ref{eq:MAE_margin} (but with $e_{\text{IPCW}}$ rather than $e_{\text{margin}}$) to compute MAE scores.

Importantly, the IPCW-T approach has downsides as well. 
IPCW weighted time is incapable of approximating the value for censored subjects with no subsequent event times. 
The same problem applies to IPCW-D and IPCW Brier score as well~\cite{graf1999assessment}, where the denominator $G(t_i)$ in Equation~\ref{eq:MAE_ipcw_diff} will equal to zero for those censored-at-last subjects.
These subjects must be excluded from the evaluation. 
This is consistent with Administrative Brier Score~\cite{kvamme2019brier}, in which individuals are removed from evaluation after their administrative censoring time.

\subsection{MAE-PO}
\label{sec:mae-po}
Both MAE-margin and MAE-IPCW-T use surrogate event values for the censored subjects; of course, this is only useful if those surrogate values are accurate.

Another way to estimate surrogate event values is using the pseudo-observations~\cite{andersen2003generalised, andersen2010pseudo}. 
Let $\{t_i\}_{i=i}^N$ be i.i.d. draws of a random variable time, $T$, and let $\hat{\theta}$ be an unbiased estimator for the event time based on right-censored observations of $T$. 
The pseudo-observation for a censored subject is defined as: 
\begin{equation}
\label{eq:PO}
    e_{\text{pseudo-obs}}(t_i, \Data) = N \times \hat{\theta} - (N-1) \times \hat{\theta}^{-i} \ ,
\end{equation}
where $\hat{\theta}^{-i}$ is the estimator applied to the $N-1$ element dataset formed by removing that $i$-th instance. 
The pseudo-observation can be viewed as the contribution of subject $i$ to the unbiased event time estimation $\hat{\theta}$.
Here we can use, the mean survival time of the KM estimator, 
$\hat{\theta} = \E_t [ \SurvKM{t}{\Data}] $ and $\hat{\theta}^{-i} = \E_t [ \SurvKM{t}{\Data^{-i}}] $ 
as unbiased estimators.
After calculating the pseudo-observation values using Equation~\ref{eq:PO}, MAE-pseudo-observation (MAE-PO) then uses the re-weighting scheme in Equation~\ref{eq:MAE_margin} to produce the overall score.

Pseudo-observation values can be treated as though they are i.i.d. (Appendix~\ref{appendix:pseudo-obs_iid}). 
\citet{graw2009pseudo} has shown that, as $N \rightarrow \infty$, the pseudo-observation can approximate the correct conditional expectation:
\begin{equation*}
    \E [\,e_{\text{pseudo-obs}}(t_i) \mid \Bfx{i}\,]
    \quad\approx\quad \E [\,e_i \mid \Bfx{i}\,] \ ,
\end{equation*}
in situations where censoring does not depend on the covariates. 
However, when comparing the empirical performance of MAE-PO, we find it also works well in the situation where censoring is dependent on the covariates, which is consistent with~\citet{binder2014pseudo}.

To be a meaningful estimate, we need to consider that these pseudo-observation values have some important properties -- in particular, the pseudo-observation value for a censored subject is always greater than the censored time (see Theorem~\ref{theorem:authenticity}).
Appendix~\ref{appendix:pseudo-obs_property} provides more details about MAE-PO and its properties.

\begin{figure*}[t]
    \centering\includegraphics[width=\textwidth]{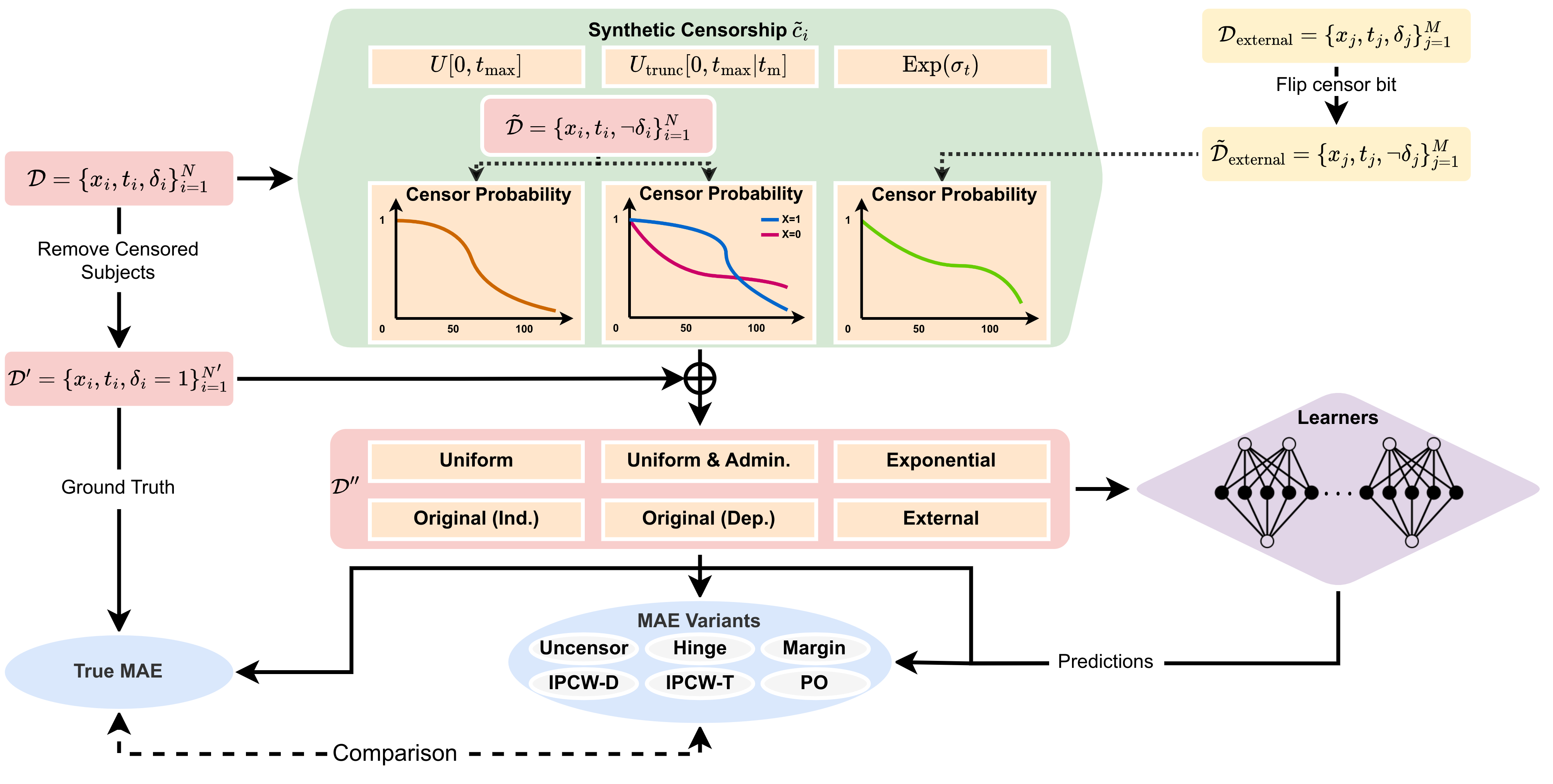}
    \vspace*{-0.2in}
    \caption{Flowchart illustrating the generation of realistic semi-synthetic survival datasets and the subsequent evaluation of MAE metrics.}
    \label{fig:flowchart}
\end{figure*}

\section{Experiments and Results}
\label{sec:results}

We conducted extensive experiments\footnote{Code to replicate all experiments can be found at \hbox{\url{https://github.com/shi-ang/CensoredMAE}}} to evaluate the effectiveness of the proposed evaluation methods -- comparing the effectiveness of these 6 evaluation metrics for estimating the actual MAE of various survival models on a wide range of survival datasets.
The two primary research objectives we care about are:
\begin{itemize}
\itemsep0em
    \item Can the metric accurately \textbf{rank} the performance of models, \ie can it identify the better-performing models?
    \item Does the performance score generated by the MAE variants closely \textbf{approximate} the true MAE?
\end{itemize}
\add[]{These two questions encompass the discrimination and calibration aspects of the metric performance, respectively.}
Since it is not possible to obtain a true MAE evaluation for any existing survival dataset, we will construct semi-synthetic datasets using real datasets with synthetic censorship. 
Of course, the algorithms for learning the survival models will only see the censored time for those synthetic-censored subjects; we will only use their true event time when we calculate the true MAE.

\subsection{Semi-Synthetic Datasets}
\label{sec:semi-synthetic_data}

To evaluate the MAE-inspired evaluation metrics, we need to know the true MAE, which means explicitly knowing when each subject will experience the event. 
As this information is not available in a real-world survival dataset, we need to produce a synthetic one.
While it is easy to make up arbitrary covariates $\Bfx{i}$, and arbitrary event time $e_i$ and censor time $c_i$, to be useful, we instead want synthetic data to be realistic, and matching the covariates, event distribution, and censoring distribution of some real-world dataset.

This motivated our approach for generating the semi-synthetic datasets (see also the flowchart in Figure~\ref{fig:flowchart}): 
\begin{enumerate}
    \item Start with a real-world survival dataset $\mathcal{D}$;
    \item Calculate some useful statistics and 
    the censoring distributions for generating the 
    censor times;
    \item Produce $\mathcal{D}'$ by removing the censored instances from $\mathcal{D}$, leaving just the uncensored ones;
    \item Form $\mathcal{D}''$ by applying some reasonable but synthetic censoring types to $\mathcal{D}'$. 
    The censoring types are based on the statistics calculated in step 2.
\end{enumerate}

We apply synthetic censoring to $\mathcal{D}'$ based on the independent censorship assumption, 
{by computing a synthetic censoring time $\tilde{c}_i$ for each subject,
and then censoring that subject if the synthetic censoring time is earlier than the event time ($\tilde{c}_i < t_i$),
and otherwise leaving it uncensored.}
We consider six different kinds of censoring distributions for generating the synthetic censor times:
\begin{itemize}
\itemsep0em
    \item Uniform distribution, $\tilde{c}_i \sim U[0, t_{\text{max}}]$ where $t_{\text{max}}$ represents the maximum event time in $\Data'$.
    
    \item Uniform distribution, augmented with administrative censoring at the median event time, formally: $\tilde{c}_i = \min\{c_i', t_{\text{median}}\}$, where $c_i' \sim U[0, t_{\text{max}}]$, and $t_{\text{median}}$ is the median time in $\Data'$.
    
    \item Exponential distribution, $\tilde{c}_i \sim \text{Exp}(\sigma_t)$, where $\sigma_t$ is the standard deviation of the event times in $\Data'$.
    \item Original censoring distribution {\em independent of the features}, $\tilde{c}_i \sim G_{\text{KM}(\tilde{\Data})}(t)$, where $G_{\text{KM}(\tilde{\Data})}$ is the censoring distribution estimated by the KM algorithm,
    with the censor-bit-flipped datasets ($\tilde{D}$ in Figure~\ref{fig:flowchart}). 
    \item Original censoring distribution {\em dependent on the features}, $\tilde{c}_i \sim G_{\text{CoxPH}(\tilde{\Data})}(t \mid \Bfx{i})$, where $G_{\text{CoxPH}(\tilde{\Data})}$ is the feature-dependent censoring distribution estimated by a Cox Proportional Hazard model (CoxPH)~\cite{cox1972regression} with Breslow estimator~\cite{breslow1975analysis}.
    \item Censoring distribution from an external (GBM) dataset, $\tilde{c}_i \sim G_{\text{KM}(\tilde{\Data}_{\text{external}})}(t)$, due to its large percentage of early censoring. 
    We rescaled the sampled censoring time, so the range matches the event times in $\Data'$.
\end{itemize}
While this is not perfect, at least we know that these resulting semi-synthetic datasets, $\mathcal{D}''$, will have many properties of a real-world survival dataset: the real-world covariates domain, close-to-reality event distribution, and (a version of) close-to-reality censoring distribution\footnote{We are aware these steps do introduce some bias to the data, but this seems unavoidable.}.

We apply this synthetic censoring to 5 real-world datasets: GBM, SUPPORT, METABRIC, MIMIC-IV~\cite{johnson2022mimic} all-cause mortality datasets (MIMIC-A) and MIMIC-IV hospital mortality datasets (MIMIC-H). 
Table~\ref{tab:data_comp} summarizes the characteristics of these five datasets, and Appendix~\ref{appendix:data_details} contains information on the dataset preprocessing and MIMIC-IV datasets construction. 

\begin{table*}[!ht]
\centering
\caption{Summary of five datasets used in the empirical comparison. }
\label{tab:data_comp}
\begin{tabular}{lrrrrc}
\toprule
Dataset  & \%Censored & \#Instances & \#Event & Max Event $t_{\text{max}}$ & \#Features$^\dagger$ \\   \midrule
GBM      & 17.65\%  & 595         & 490      & 3,881      &  \; 8 (10)                   \\
SUPPORT  & 31.89\%  & 9,105       & 6,201    & 1,944      & 26 (31)                   \\
METABRIC & 42.07\%  & 1,904       & 1,103    & 355        & 9 (9)                    \\
MIMIC-IV (all-cause mortality) & 66.65\%  & 38,520      & 12,845   & 4,404      & 91 (91)           \\
MIMIC-IV (hospital mortality) & 97.69\%  & 293,907     & 6,780    & 248        & 10 (10)         \\
\bottomrule
\end{tabular}
\\
$^\dagger$ The number of features before performing one-hot encoding, with brackets (the number of features after one-hot encoding).\\
\end{table*}

\subsection{Models}
We compare the time-to-event prediction results using 10 survival prediction models: 
a naive linear regressor,
KM~\cite{kaplan1958nonparametric}, 
CoxPH (using an extension to produce an ISD)~\cite{cox1972regression}, 
Accelerate Failure time (AFT)~\cite{stute1993consistent} with the Weibull distribution,
Gradient Boosting Machine with component-wise least squares (GBM-C)~\cite{hothorn2006survival},
Random Survival Forest (RSF)~\cite{ishwaran2008random},
Multi-Task Logistic Regression (MTLR)~\cite{yu2011learning, jin2015using},
DeepHit~\cite{lee2018deephit},
Survival Cluster Analysis (SCA)~\cite{chapfuwa2020survival},
and Survival Mixture Density Network (S-MDN)~\cite{Han2022SurvivalMD}.
Appendix~\ref{appendix:model_details} describes these 10 models, including the implementation details and hyperparameter settings -- and how the non-survival prediction models dealt with censoring training instances.

All models are trained on the semi-synthetic datasets, $\Data''$. We split the data into a training set (80\%) and a test set (20\%) using a stratified 5-fold cross-validation (5CV) procedure (stratified wrt both time $t$ and censor indicator $\delta$). 
If the model requires a validation set for hyper-parameter tuning or early stopping, we will split 20\% of the training set as the validation set.
We then compute the mean on all the evaluation metrics across the 5CV folds.

\subsection{Evaluation Metrics}
We will use all six MAE metrics described in Section~\ref{sec:mae_censor} to measure the prediction error between the synthetic censoring times and estimated times. 
We will also compute the true MAE score as the ground truth, using the hidden-to-learner true event times.

\subsection{Experimental Results}
We have 5 clinical datasets with 6 types of synthetic censoring, therefore there are in total $5 \times 6 - 1$%
\footnote{The GBM dataset with external (GBM dataset) censoring will just be the same as feature-independent original censorship.} 
semi-synthetic datasets in our experiments. 
Due to the space limit, the main text only reports the results on three datasets (GBM, METABRIC, and MIMIC-A) and three censoring distributions (uniform with administrative censoring, feature-independent original censorship, and feature-dependent original censorship); see Figure~\ref{fig:results_selected}.
Appendix~\ref{appendix:results_full} presents the results of all semi-synthetic datasets
-- all combinations of datasets and censoring distributions.

\begin{figure*}[!ht]
    \centering\includegraphics[width=\textwidth]{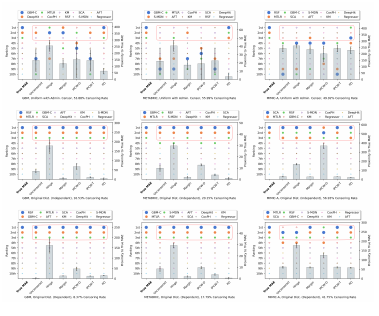}
\vspace*{-0.5in}
    \caption{Evaluation metrics comparison in terms of ranking accuracy (left axis) and proximity to true MAE (right axis). Each row refers to a specific censoring type (uniform with administrative censoring, feature-independent original censoring, and feature-dependent original censoring), and each column to a specific dataset (GBM, METABRIC, MIMIC-A). }
    \label{fig:results_selected}
\end{figure*}

Each subplot in Figure~\ref{fig:results_selected} 
(and also the figures in Appendix~\ref{appendix:results_full}) 
corresponds to a semi-synthetic dataset.
Within each subplot, the first column is the ranking performance evaluation using the 5CV mean of true MAE (ground truth).
{Survival prediction models with larger circles had
better true MAE.}
{The gray bar plots on columns other than the first show
how close that MAE-inspired metric is to the true MAE.
If the error is large, we may still want to know which of the metrics is best at identifying which of the survival prediction models is best, by having a ranking of the models that agrees with the true ranking (for instance by identifying the three best models; see the red rounded-box at the top).}

In the top left plot (GBM dataset with uniform and administrative censorship), we see that the GBM-C model was the most accurate, as it is represented by the largest blue solid circle, followed by DeepHit (second largest orange solid circle), then MTLR (solid green circle), then the 7 models with smaller open circles. They appear, in descending order, in the far left column labeled ``True MAE''. The other columns show how well the various variants of MAE do, in terms of approximating this true MAE.

The red ``rounded box'' at the top of the plot shows each metric's preference over the models.
Here, a concordant preference to true MAE's indicates a great performance. We can see that PO (pseudo-observation) top three choices were GBM-C, DeepHit, and MTLR -- which exactly matched the truth (first column by true MAE). 
By contrast, Margin had DeepHit first (not second), GBM-C second (not first), then MTLR in position 5 (not third). Other models did even worse at matching the order. 
Now consider the gray vertical bars, which show how close that MAE-inspired metric was to the true MAE, over these 10 different learning models. Here, smaller values mean that measure did well. We see that the far right PO was the smallest and with fairly tight error bars. Margin was second best, then Uncensored, IPCW-D, and IPCW-T, essentially tied, though the latter showed smaller variation, and with Hinge coming in last.

To demonstrate the effectiveness of the surrogate time method in MAE-margin, MAE-IPCW-T, and MAE-PO values, we also perform an ablation study that used the mean survival time of the KM estimation on the whole group as the proxy event time; this is the MAE population pseudo-observation (MAE-Pop-PO), see Appendix~\ref{appendix:results_full}.

\subsubsection{Uniform with Administrative Censorship}
The three subplots in the first row in Figure~\ref{fig:results_selected} demonstrate the metrics performance for uniform censoring distributions with administrative censoring. 
The MAE-PO is the best here, in both ranking performance (as it is the only one that correctly identifies the top-three models in GBM and METABRIC, and the only one that identifies two of the top-three models for MIMIC-A) and closeness to the true MAE (significantly better in GBM and METABRIC with $p$-value $<$ 0.05 via $t$-test, and one of the best in MIMIC-A).

MAE-margin is the runner-up as it can identify parts of the best-performing models and has the second closest difference to true MAE. 
Between IPCW-D and IPCW-T, the performance does not have a significant difference in both ranking and proximity to true MAE.  
However, IPCW-D is associated with quite large error bars, which may be because the accuracy of later uncensored subjects will dominate the score (as we discussed in Section~\ref{sec:ipcw-d}).

\subsubsection{Feature-Independent Original Censorship}
The three subplots in the second row in Figure~\ref{fig:results_selected} demonstrate the metrics performance for feature-independent original censoring distribution.
Among all the evaluation metrics, margin and pseudo-observation perform equally the best for identifying the top-three performing models. 
MAE-pseudo-observation has a slight advantage in the proximity to true MAE, as its value is closer to the true MAE on GBM and significantly closer ($p$-value $<$ 0.05) on METABRIC and MIMIC-A datasets. 
In addition, we also observed that IPCW-D is always associated with a large variance (reason explained above). 

\subsubsection{Feature-Dependent Original Censorship}
The three subplots in the third row in Figure~\ref{fig:results_selected} demonstrate the metrics performance for feature-independent original censoring distribution.
MAE-uncensored 
performs the best on GBM datasets,
which may 
be due to the low synthetic censoring rate of this dataset (8.37\% censoring rate), 
meaning 
the whole dataset could be approximately represented by the uncensored population.
Among the MAE metrics that can handle the censored subjects, MAE-margin, IPCW-T, and MAE-PO perform equally well on GBM and MIMIC-A datasets.
For the METABRIC dataset, pseudo-observation is the best 
metric as it has the significantly lowest error to true MAE among all the metrics that can identify the top-three performing models.  

\subsubsection{Other Types of Censoring Distributions}
Figures~\ref{fig:uniform_full}, \ref{fig:exp_full}, and \ref{fig:gbm_full} in Appendix~\ref{appendix:results_full} demonstrate the performance metrics for uniform censoring, exponential censoring, and GBM censoring distribution. 
In all 14 cases (5 for uniform, 5 for exponential, and 4 for GBM censoring), 
MAE-PO is the best 
10 times (71\%) while MAE-margin is the 
best for the other four. 
The MAE-margin metric prevails in 3 out of 5 datasets with exponential censoring, indicating that its performance is either superior or comparable to MAE-PO for this specific censorship type. 

\begin{table}[!t]
\centering
\caption{Summarization of MAE-based metric performance by counting the number of times each metric is best. *Note this includes ties.}
\label{tab:summary}
\vspace{1pt}
\setlength\tabcolsep{2.5pt}
\resizebox{\columnwidth}{!}{\begin{tabular}{lccccccc}
\toprule
         & Uni. & Uni.\&Admin. & Exp. & Orig.(Ind.) & Orig.(Dep.) & GBM & Total \\
\midrule
Uncensor & 0    & 0            & 0    & 0           & 1           & 0   & 1     \\
Hinge    & 0    & 0            & 0    & 0           & 0           & 0   & 0     \\
Margin   & 2*   & 0            & 3    & 2*          & 1           & 0   & 8*    \\
IPCW-D   & 0    & 0            & 0    & 0           & 0           & 0   & 0     \\
IPCW-T   & 0    & 0            & 0    & 0           & 0           & 0   & 0     \\
PO       & 4*   & 5            & 2    & 4*          & 3           & 4   & 22*   \\
\bottomrule
\end{tabular}}
\end{table}

Table~\ref{tab:summary} presents a comprehensive summary of the performance of all MAE-based metrics across 29 semi-synthetic datasets. 
Each column in the table corresponds to a specific censorship type, as outlined in Section~\ref{sec:semi-synthetic_data}. 
The values within the table indicate the number of times that each metric outperforms the others for a given censorship distribution. 
The final columns provide an overview of the cumulative results from all experiments. 
Notably, the MAE-PO metric demonstrates its robustness by emerging as the superior choice in 22 out of 29 experiments, accounting for a 76\% success rate. 
MAE-margin is best for most of the other ones, especially when the censoring is exponentially distributed.
As a result, we recommend that researchers consider using MAE-margin for datasets that seem to exhibit an exponential censoring distribution, and using MAE-PO for datasets with other types of censoring.

\section{Discussion and Conclusion} 

Here, we first argue that MAE should be used as 
an 
evaluation metric for 
evaluating survival prediction models (\eg ISDs),
especially for the standard such task, which requires predicting time-to-event.
Appendix~\ref{sec:six_metrics} shows that 
MAE is a {\em proper} scoring rule (for an uncensored dataset),
and that no other metric quantifies the time-to-event accuracy like MAE does.
To handle the right censorship, we introduced three novel MAE variations: MAE-IPCW-D, MAE-IPCW-T, and MAE-PO, and empirically 
compared 
them to the three existing variants: MAE-uncensored, MAE-hinge, and MAE-margin,
over several realistic semi-synthetic datasets.
These empirical results 
demonstrate that MAE-PO
($i$)~can often correctly identify the top-performing models
and ($ii$)~often has error closest to true MAE.

We recognize that MAE-PO method has certain limitations, particularly subject to the limitation of KM, \eg ($i$) not accounting for the effects of covariates, and ($ii$) the requirement for the independent censoring assumption to be valid.


This research focuses on only the MAE score. 
There are, however, many other ways to measure the errors between the predicted and observed time. 
For instance, one can use the Mean Squared Error (MSE) to penalize large prediction errors, relative MAE or (relative MSE) score to compare models where errors are measured in different scales, or normalized versions to set an upper constraint on the MAE or MSE score. 
Importantly, we can easily 
apply all of the techniques mentioned above
to handle censored subjects to these variations.
{Note also that we can use the realistic semi-synthetic datasets defined above along with the proposed methodology to evaluate other evaluation metrics.}

Following the famous remark
``All models are wrong, but some are useful''~\cite{box1976science},
it is important to precisely define ``usefulness''.
For many clinical tasks, where this corresponds to
time-to-event, it is important to have measures that 
can determine if a model can accurately estimate event time values.
This paper has provided such a measure (MAE-PO) for survival prediction, and also demonstrated that it works effectively. Note this also required finding ways to produce realistic semi-synthetic datasets.
We anticipate others will be able to use this methodology for evaluating evaluation measures,
and more importantly, for the result of this analysis, suggesting that MAE-PO often approximates the true MAE.


\section*{Acknowledgements}
This research received support from the Natural Science and Engineering Research Council of Canada (NSERC), the Alberta Machine Intelligence Institute (Amii), NIH/NIDDK R01-DK123062, NIH/NINDS R61-NS120246, NIH/NHLBI Award R01HL148248, NSF Award 1922658 NRT-HDR: FUTURE Foundations, Translation, and Responsibility for Data Science, and NSF CAREER Award 2145542. The authors extend their gratitude to the anonymous reviewers for their insightful feedback and valuable suggestions.


\nocite{langley00}

\bibliography{paper}
\bibliographystyle{icml2023}

\newpage
\appendix
\onecolumn
\section{Overview of the Appendix}
Appendix~\ref{appendix:notations} provides a summary of the notation and assumptions used throughout the study.
Appendix~\ref{sec:six_metrics} compares the MAE metric with five other commonly used evaluation metrics. 
Appendix~\ref{appendix:pseudo-obs_property} theoretically describes the property of pseudo-observation and proves its authenticity.
Appendix~\ref{appendix:exp_details} describes the implementation details, including the datasets and the ten models used to estimate the timing of events. 
Appendix~\ref{appendix:results_full} includes the complete results for MAE-inspired evaluation metrics comparison on the 29 semi-synthetic datasets.

\section{Notation}
\label{appendix:notations}
In general, a survival dataset contains \textit{N} time-to-event tuples, $\mathcal{D}=\{(\Bfx{i}, t_i, \delta_i)\}_{i=1}^N$, where $\Bfx{i} \in \mathbb{R}^d$ represents the observed $d$-dimensional features for the $i$-th instance, $t_i \in \mathbb{R}_+$ denotes the event or censor time, and $\delta_i \in \{0, 1\}$ is a censor/event indicator where $\delta_i = 0$ means the subject is right-censored (the subject has not experienced an event) and $\delta_i = 1$ means observed event times. We assume each patient has a event time $e_i$, and a censoring time $c_i$, and assign $t_i \leftarrow \min\{e_i, c_i\}$ and $\delta_i \leftarrow \mathbbm{1}[e_i \leq c_i]$.

\paragraph{Assumption} \textit{(Independent censoring)}
In this study, we follow the standard convention that event time and censor time are assumed independent and conditional on the covariates. Formally, for event time $e_i \sim E$, censoring time $c_i \sim C$ and covariates $\Bfx{i} \sim \mathbf{X}$, we assume $E \ \bot \ C \mid \mathbf{X}$.

Random censoring is another commonly used assumption with the definition $T \ \bot \ C$ without conditioning on $\mathbf{X}$. 
However, since random censoring implies independent censoring,  
all the nature and properties proved under the independent assumption will also hold for the random assumption.


The predicted ISD curves can also serve as a surrogate for the risk scores as well. 
For instance, the time-independent risk scores can be defined as the negative value of the predicted survival times of ISDs. 
Alternatively, we can define the time-dependent risk scores, at time $t > 0$, as the negative of the survival probability, $-\Surv{t}{\Bfx{i}}$.

We include a notation table, Table~\ref{tab:notation}, to summarize the symbols and abbreviations we used in the paper.
\begin{table}[!ht]
\centering
\caption{Table of notation. Ordered alphabetically.}
\label{tab:notation}
\begin{tabular}{ll}
\toprule
Symbol/Abbreviation            & Definition \\ \midrule
$c_i$                   & Censor time of subject $i$       \\
$\tilde{c}_i$           & Synthetic censor time of subject $i$       \\
$\Data$                 & Raw Dataset                           \\
$\Data'$                & Raw Dataset with only uncensored subjects      \\
$\Data''$               & Semi-synthetic Dataset with synthetic censoring on $\Data'$     \\
$\tilde{\Data}$         & Raw Dataset with flipped censor bit      \\
$\text{Exp}(\lambda)$   & Exponential distribution with $\lambda$ as the parameter     \\ 
${\mathbb E}[\cdot]$ & Expectation \\
$e_i$                   & Event time of subject $i$       \\
$F(t\mid \Bfx{i})$      & Cumulative density function given the covariates $\Bfx{i}$    \\
$f(t\mid \Bfx{i})$      & Probability density function given the covariates $\Bfx{i}$    \\
$G(t)$                  & Censor distribution       \\
$G_{\text{KM}}(t)$      & Feature-independent censor distribution, estimated using KM model       \\
$G_{\text{CoxPH}}(t\mid \Bfx{i})$ & Feature-dependent censor distribution, estimated using CoxPH model       \\
$N$                     & The size of the dataset, number of subjects      \\
$\mathcal{R}$           & Scoring rule       \\
$S_m(t\mid \Bfx{i})$      & Survival distribution given the covariates, estimated by model $m$      \\
$\SurvKM{t}{\Data}$     & Group-level survival distribution, estimated using KM model on the dataset $\Data$      \\
$t_i$                   & Observed time of subject $i$, $t_i = e_i$ if $\delta_i = 1$ and $t_i = c_i$ if $\delta_i = 0$      \\
$\hat{t}_i$             & Predicted event time for subject $i$  \\
$U[a, b]$               & Uniform distribution starting at $a$ and ending at $b$      \\
$\Bfx{i}$               & Covariates of subject $i$       \\
$\delta_i$              & Censor bit / censor indicator, $\delta_i = \mathbbm{1}_{e_i > c_i}$      \\
$\hat{\theta}$          & Unbiased estimator for the event time       \\ 
$\omega_i$              & Confidence weights for subject $i$       \\ \midrule
1-calibration           & Hosmer-Lemeshow Calibration           \\
BS                      & Brier Score                           \\
C-index                 & Concordance Index                     \\
CV                      & Cross-Validation                      \\
D-calibration           & Distribution Calibration              \\
IBS                     & Integrated Brier Score                \\
IPCW                    & Inverse Probability Censoring Weight  \\
IPCW-D                  & Inverse Probability Censoring Weight on Difference  \\
IPCW-T                  & Inverse Probability Censoring Weight on Time  \\
ISD                     & Individual Survival Distribution      \\
KM                      & Kaplan-Meier estimator                \\
LL                      & Log-Likelihood                        \\
MAE                     & Mean Absolute Error       \\
PDF                     & Probability Density Function          \\
PO                      & Pseudo-Observation                    \\
\bottomrule
\end{tabular}
\end{table}

\section{Comparing MAE with Other Evaluation Metrics}
\label{sec:six_metrics}
Survival prediction is like regression as it predicts a real number (the subject's event time) from a description of a patient $\Bfx{i}$. Given this commonality, we want to evaluate survival prediction models using measures for evaluating regression tasks, such as mean absolute error (MAE) or mean squared error (MSE).

MAE measures the mean of the absolute difference between the true event time and the predicted time.
As suggested in Figure~\ref{fig:6_metrics} (a) the MAE score for an uncensored subject is (reformulated Equation~\ref{eq:MAE_uncensor})
\begin{equation}
\label{eq:MAE_uncensor_appendix}
    \mathcal{R}_{\text{MAE}} (\, \SurvSub{m}{\cdot}{\Bfx{i}}, \, t_i, \, \delta_i = 1 \,) \, = \, |\,t_i - \hat{t}_i\,| ,
\end{equation}
where $\hat{t}_i$ is the median survival time from Equation~\ref{eq:median_survival_time}\footnote{Alternatively, we could use the mean survival time for $\hat{t}_i$ from Equation~\ref{eq:mean_survival_time}.}. MAE is a negative scoring rule which means the smaller the loss, the better the model performs.

\begin{theorem}
\label{them:MAE_median}
    The MAE score for uncensored subjects, $\mathcal{R}_{\text{MAE}} (\, \SurvSub{m}{t}{\Bfx{i}}, \, t_i, \, \delta_i = 1 \,)$, is a proper scoring rule if we use median survival time of the predicted ISD as the predicted time.
\end{theorem}
\begin{proof}
By the proper scoring rule definition proposed by \cite{gneiting2007strictly} and \cite{rindt2022survival}, a negative scoring rule (a lower score indicates better performance) is proper if for any model $m$ we have
\begin{equation*}
    \mathbb{E}_{i \sim \mathcal{D}} \ \mathcal{R}(\, \SurvTrue{t}{\Bfx{i}}, \, t_i, \, \delta_i \,) \, \leq \,  \mathbb{E}_{i \sim \mathcal{D}} \ \mathcal{R}(\, \SurvSub{m}{t}{\Bfx{i}}, \, t_i, \, \delta_i \,) \ ,
\end{equation*}
where $\SurvTrue{t}{\Bfx{i}}$ is the true ISD distribution for each subject. Similarly, a positive scoring rule will change the above inequality from $\leq$ to $\geq$.

For an uncensored dataset, every subject has the observed time equal to the event time ($t_i = e_i$). 
We need to prove that, for any $e_i \sim \SurvTrue{\cdot}{\Bfx{i}}$, for any $\Bfx{i} \sim \mathbf{X}$, using the median survival time of the true distribution will minimize the MAE. We can formulate the MAE score by: 
\begin{equation*}
\begin{aligned}
    \mathbb{E}_{\Bfx{i}, t_i \sim \mathcal{D}} \  [\mathcal{R}_{\text{MAE}}(& \, \SurvSub{m}{\cdot}{\Bfx{i}},  \, t_i, \, \delta_i = 1 \,) \mid \delta_i=1] \\
    &= \mathbb{E}_{\Bfx{i} \sim \mathbf{X}, e_i \sim \SurvTrue{e_i}{\Bfx{i}}} \left[\left| e_i - \hat{t}_i  \right|\right] \\
    &=  \int_{\Bfx{i} \sim \mathbf{X}} p(\Bfx{i}) \int_0^{\infty} f_{\text{true}}(e_i \mid \Bfx{i})\left|e_i - \hat{t}_i \right| \ de_i \ d\Bfx{i} \\
    &=  \int_{\Bfx{i} \sim \mathbf{X}} p(\Bfx{i}) \left( \int_0^{\hat{t}_i} f_{\text{true}}(e_i \mid \Bfx{i}) (\hat{t}_i - e_i) \ de_i  +  \int_{\hat{t}_i}^{\infty} f_{\text{true}}(e_i \mid \Bfx{i}) (e_i - \hat{t}_i ) \ de_i \right) d\Bfx{i} \ ,
\end{aligned}
\end{equation*}
where $f_{\text{true}}(e_i \mid \Bfx{i})$ represents the true PDF function, and $p(\Bfx{i})$ represents the marginal over covariates. Computing the derivative of the above equation and setting it to zero will allow us to identify the minimum of this function with respect to $\hat{t}_i$. 
By taking the derivative:
\begin{equation*}
\begin{aligned}
    & \frac{d}{d \hat{t}_i} \int_{\Bfx{i}} p(\Bfx{i}) \left( \int_0^{\hat{t}_i} f_{\text{true}}(e_i \mid \Bfx{i}) (\hat{t}_i - e_i) \ de_i + \int_{\hat{t}_i}^{\infty} f_{\text{true}}(e_i \mid \Bfx{i}) (e_i - \hat{t}_i ) \ de_i \right) d\Bfx{i} \\
    &= \int_{\Bfx{i}} p(\Bfx{i}) \left(f_{\text{true}}(\hat{t}_i \mid \Bfx{i}) (\hat{t}_i - \hat{t}_i ) + \int_0^{\hat{t}_i} f_{\text{true}}(e_i \mid \Bfx{i}) \ de_i - f_{\text{true}}(\hat{t}_i \mid \Bfx{i}) (\hat{t}_i - \hat{t}_i ) - \int_{\hat{t}_i}^{\infty} f_{\text{true}}(e_i \mid \Bfx{i}) \ de_i \right) d\Bfx{i} \\
    &= \int_{\Bfx{i}} p(\Bfx{i}) \left( \int_0^{\hat{t}_i} f_{\text{true}}(e_i \mid \Bfx{i}) \ de_i  -  \int_{\hat{t}_i}^{\infty} f_{\text{true}}(e_i \mid \Bfx{i}) \ de_i \right) d\Bfx{i} \ ,
\end{aligned}
\end{equation*}
where the first equality holds by applying the Leibniz integral rule. Setting the above derivation to zero leads to $\hat{t}_i$ such that
\begin{flalign*}
    \int_0^{\hat{t}_i} f_{\text{true}}(e_i \mid \Bfx{i}) \ de_i  = \int_{\hat{t}_i}^{\infty} f_{\text{true}}(e_i \mid \Bfx{i}) \ de_i 
    \ \ \Longrightarrow \ \ 1 - \SurvTrue{\hat{t}_i}{\Bfx{i}} = \SurvTrue{\hat{t}_i}{\Bfx{i}} \ .  
\end{flalign*}

Therefore, the MAE is minimized if $\SurvTrue{\hat{t}_i}{\Bfx{i}} = \frac{1}{2}$, that is, when $\hat{t}_i$ is the median time of the true ISD distribution. This completes the proof.
\end{proof}
The preceding derivation uses the median survival time and MAE to demonstrate the properness of this combination. 
We can prove that mean survival time with MSE is also proper following the same line of reasoning.

\begin{theorem}
\label{them:MSE_mean}
    The MSE score for uncensored subjects, $\mathcal{R}_{\text{MSE}} (\, \SurvSub{m}{t}{\, \Bfx{i}}, \, t_i, \, \delta_i = 1 \,)$, is a proper scoring rule if we use mean survival time as the predicted time.
\end{theorem}
\begin{proof}
Follow the logic in Theorem~\ref{them:MAE_median}, we can formulate the uncensored MSE score as: 
\begin{equation*}
\begin{aligned}
    \mathbb{E}_{\Bfx{i}, t_i \sim \mathcal{D}} \ \left[ \mathcal{R}_{\text{MSE}}(\, \SurvSub{m}{t}{\Bfx{i}}, \, t_i, \, \delta_i = 1 \,) \mid \delta_i = 1 \right]
    &= \mathbb{E}_{\Bfx{i} \sim \mathbf{X}, e_i \sim \SurvTrue{e_i}{\Bfx{i}}} \left[( \hat{t}_i - e_i )^2 \right] \\
    &=  \int_{\Bfx{i} \sim \mathbf{X}} p(\Bfx{i}) \int_0^{\infty} f_{\text{true}}(e_i \mid \Bfx{i}) (\hat{t}_i - e_i)^2  \ de_i \ d\Bfx{i} \ .
\end{aligned}
\end{equation*}
We also compute the derivative of the above equation with respect to $\hat{t}_i$ and set the derivation to zero:
\begin{equation*}
\begin{aligned}
    \frac{d \ \mathbb{E}_{\Bfx{i}, e_i \sim \mathcal{D}} \ \left[ \mathcal{R}_{\text{MSE}}(\, \SurvSub{m}{t}{\Bfx{i}}, \, e_i, \, \delta_i = 1 \,) \mid \delta_i=1 \right] }{d \hat{t}_i}
    &= \frac{d}{d \hat{t}_i} \int_{\Bfx{i}} p(\Bfx{i}) \int_0^{\infty} f_{\text{true}}(e_i \mid \Bfx{i}) (\hat{t}_i - e_i)^2  \ de_i \ d\Bfx{i} \\
    &= \int_{\Bfx{i}} p(\Bfx{i}) \int_0^{\infty} 2 f_{\text{true}}(e_i \mid \Bfx{i}) (\hat{t}_i - e_i)  \ de_i \ d\Bfx{i} \ .
\end{aligned}
\end{equation*}
Setting the above derivation to zero leads to $\hat{t}_i$ such that
\begin{flalign*}
    \hat{t}_i \cdot \int_0^{\infty} f_{\text{true}}(e_i \mid \Bfx{i}) \ de_i  = \int_0^{\infty} e_i \cdot f_{\text{true}}(e_i \mid \Bfx{i})  \ de_i \ \ 
    \Longrightarrow \ \ \hat{t}_i = \int_0^{\infty} \SurvTrue{e_i}{\Bfx{i}} \ de_i \ .  
\end{flalign*}
Therefore, the MSE is minimized if $\hat{t}_i$ is the mean time of the true ISD distribution. Then the proof is complete.
\end{proof}

The main text compared the MAE to five other frequently used metrics for survival prediction (from both discriminative and calibration perspectives). 
It is necessary to have a thorough grasp of the limitations of all six metrics when selecting metrics for model optimization, and separately for model evaluation. 
We suggest that a task-oriented strategy is needed for selecting the right evaluation metrics for the application.
However, the following section argues that MAE loss is the best option if the objective is to quantify the time-to-event accuracy or to make tailored clinical decisions.

Note that the MAE methods (and also C-index) require an ISD model to use a single value as the predicted time for an instance.
In this context, the obvious candidates are either median or mean time (Equations~\ref{eq:mean_survival_time} and~\ref{eq:median_survival_time}) of the survival curves. 
However, in practice, the predicted survival curves from many ISD models often terminate at a specific time (typically the largest observed time in the training dataset) with non-zero probabilities. 
This limitation poses a challenge as the subsequent ISD distribution is unknown/censored, rendering the accurate computation of mean or median time impossible.
In this study, we adopt a linear extrapolation method to extend survival curves, following the approach outlined by~\hbox{\citet{haider2020effective}}.
The extrapolation of survival curves remains an active area of research, and we aim to explore the impact of various extrapolation techniques on MAE evaluation in future investigations.

\subsection{Concordance Index}
The C-index gauges the model's ability to rank the subject's risk.
It is described as the proportion of all comparable pairs where the predictions and outcomes are concordant.
A pair is comparable if we can determine who has the event first.
The C-index is defined by~\citet{harrell1996multivariable}:
\begin{equation*}
\begin{aligned}
    \text{C-index} 
    &= \frac{\sum_{i,j} \mathbbm{1}_{t_i < t_j} \cdot \mathbbm{1}_{\eta_i > \eta_j} \cdot \delta_i }{\sum_{i,j} \mathbbm{1}_{t_i < t_j} \cdot \delta_i } \ ,
\end{aligned}
\end{equation*}
where $\eta_i$ represents the risk score of subject $i$. 
As we discussed in Section~\ref{sec:notation}, the risks can be either defined as the negative of expected time or of survival probability at some specified time. If we use the negative of predicted time, it is called time-independent C-index ($C^{ti}$). 
A visual illustration of $C^{ti}$ using a discordant pair can be found in Figure~\ref{fig:6_metrics} (b), which assesses if the order of true event times (stars) is concordant with the order of predicted event times (triangles).

As many researchers pointed out \cite{antolini2005time}, $C^{ti}$ is known to be biased upwards if the amount of censoring in the data is high (also proved in Proposition~\ref{them:c-index_bound}). 
\add[R3]{Therefore, }\citet{antolini2005time} proposes a time-dependent C-index ($C^{td}$) that claims to solve the issue. 
It is essentially a weighted average of the time-dependent area under the curve (AUC)\footnote{\add[R3]{In datasets with binary outcomes and no censoring, the AUC and C-index are equivalent. In survival settings, $C^{ti}$ can be considered as a generalization of the AUC, as it uses the observed times to construct pairs and predicted events to compare the concordance, whereas the AUC only uses binary statuses and predicted probabilities at a specific time point, respectively.
}} scores over time. 
Unfortunately, this does not solve the issue, as $C^{td}$ is not a proper scoring rule, proved by~\citet{rindt2022survival}. Note that $C^{ti}$ and AUC scores are also not proper, using the same reasoning.

\begin{proposition}
\label{them:c-index_bound}
Given a dataset with a censoring rate of $1 - a$. To evaluate the concordance index, the ratio of comparable pairs to total pairs is bounded by:
\begin{equation}
    a^2\ \leq\ r\ \leq\ 2a - a^2 \ ,
\end{equation}    
where $r$ is the comparable-to-total pairs ratio.
\end{proposition}

\begin{proof}
For a dataset with $n$ instances, the total number of pairs is $C(n, 2) = \frac{n(n-1)}{2}$ and the number of comparable pairs is varied based on the temporal position of censored and event instances. In the two extreme cases, the minimum number of comparable pairs happens when all censored instances happened before the earliest event instance (none of the censor-event pairs is comparable), while the maximum number of comparable pairs happens when all censored instances happened after the last event instance (arbitrary censor-event pair is comparable). Therefore, the number of comparable pairs is bounded by:
\begin{equation*}
    C(n\times a, 2) \leq \text{number of comparable pairs} \leq C(n\times a, 2) + n a \times (n - n a) \ . 
\end{equation*}
Then the ratio of comparable pairs of total pairs (by dividing the above equation by the total number of pairs $\frac{n(n-1)}{2}$) is bounded by:
\begin{equation*}
    \frac{a(na-1)}{n-1}\ \leq\ r\ \leq\ \frac{a(na-1)}{n-1} + \frac{2na(1-a)}{n-1} \ .
\end{equation*}
{As the dataset size grows}
($n\to \infty$), the bound becomes:
\begin{flalign*}
    \lim_{n\to\infty}\frac{a(na-1)}{n-1} \ \leq\ \; r \leq \lim_{n\to\infty} \frac{2na-a-na^2}{n-1}  \ \ \
    \Longrightarrow \ \ a^2 \leq \; r \leq  2a - a^2  \ ,
\end{flalign*}
then the proof is complete.
\end{proof}

\begin{figure}[t]
    \centering
    \includegraphics[width=0.5\columnwidth]{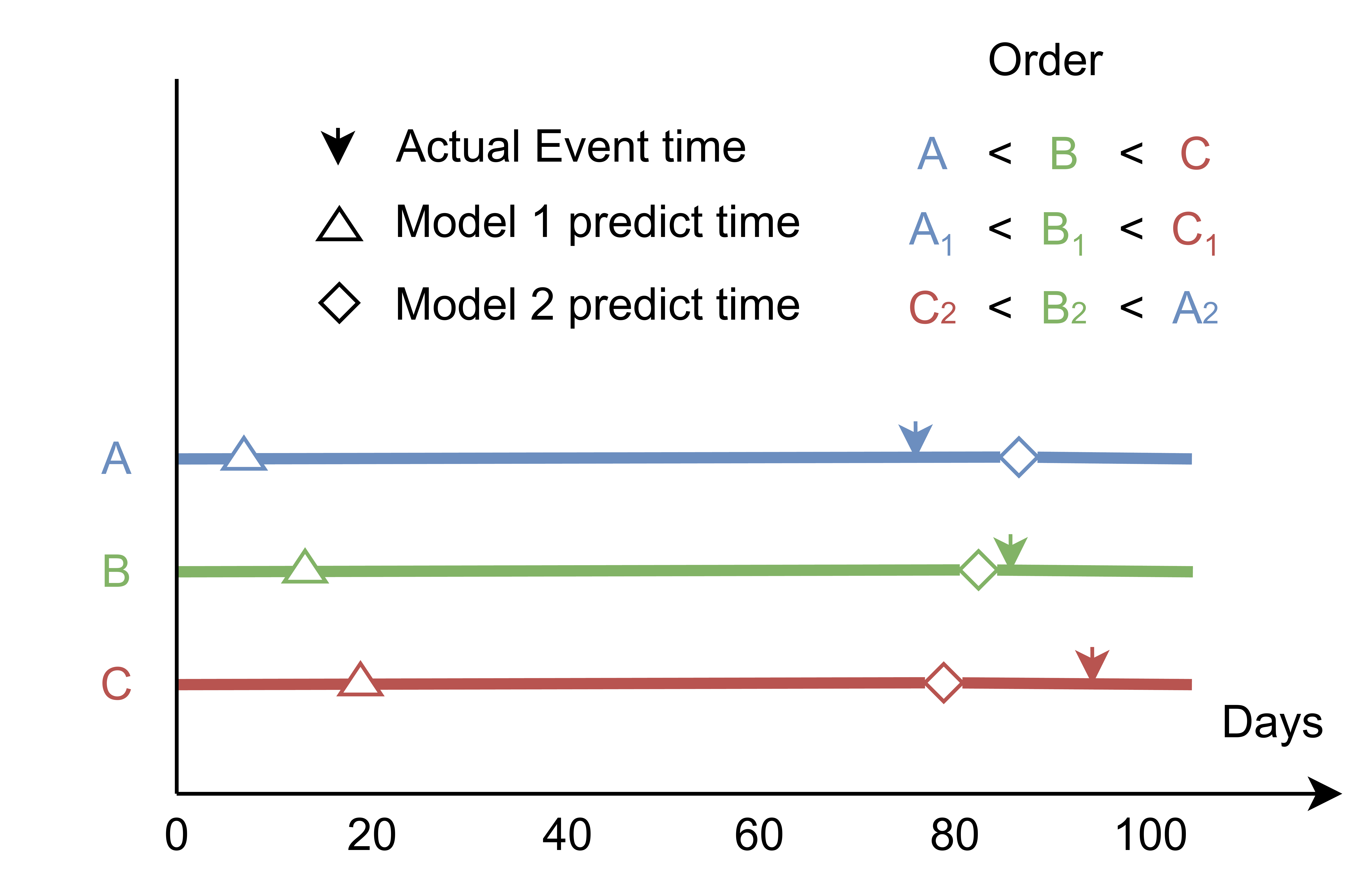}
    \caption{Comparison between MAE and Concordance Index. The arrows indicate the actual event times, and the triangles and diamonds indicate the predicted event times using Model 1 and Model 2, respectively.}
    \label{fig:mae_vs_cindex}
\end{figure}

It would be tempting to assume that MAE and C-index favor the models monotonically, \ie if Model 1 has a lower MAE than Model 2, then Model 1 must have a higher C-index and vice versa.
This is not always the case.
Figure~\ref{fig:mae_vs_cindex} shows three patients who die at 77 days (A), 85 days (B), and 93 days (C).
We can see from the model's predicted event times that Model 1 perfectly ranks the events (C-index = 1), but its predicted event times are far from the actual event times. 
While Model 2 predicts the order of the events in the wrong order (C-index = 0), it has a more accurate time estimation.
This illustration shows that MAE preferences do not coincide with C-index preferences over the model.
This is because MAE estimation focuses on measuring the discriminative accuracy at the individual level, while C-index measures the discriminative accuracy at the pairwise level.

\subsection{Integrated Brier Score}
The integrated Brier score (IBS) is the expectation of single-time Brier scores (BS)~\cite{graf1999assessment} over time. 
For each time $t$, BS calculates the squared difference between the model's predicted probability and the status (0 or 1) of an uncensored subject at that time. 
For subjects censored before time $t$, BS will use the inverse probability censoring weight (IPCW) technique to uniformly transfer their weights to subjects with known status at that time.
IBS for survival prediction is typically defined as:
\begin{equation*}
\begin{aligned}
    \mathcal{R}_{\text{IBS}} (\SurvSub{m}{\cdot}{\Bfx{i}}, t_i, \delta_i) = \frac{1}{t_{\text{max}}}  \cdot \int_0^{t_{\text{max}}} \frac{\SurvSub{m}{t}{\Bfx{i}}^2 \cdot \mathbbm{1}_{t_i \leq t, \delta_i=1}}{G(t_i)} +  \frac{(1 - \SurvSub{m}{t}{\Bfx{i}})^2 \cdot \mathbbm{1}_{t_i > t}}{G(t)} \ dt \ ,
\end{aligned}
\end{equation*}
where $t_{\text{max}}$ is normally the maximum event time of the combined training and validation datasets. 
$G(t)$ is the non-censoring probability at time $t$, which is typically estimated with KM, and its reciprocal is referred to as the IPCW weights (refers to Equation~\ref{eq:MAE_ipcw_diff}). 
The principle of IPCW for IBS is to evenly and repetitively distribute the weight of a censored subject to subjects who will experience the event after its censored time~\cite{graf1999assessment, vock2016adapting}.
Figure~\ref{fig:6_metrics} (c) provides a graphic depiction of IBS for an uncensored subject.
IBS score is represented by the weighted squared error of the shaded regions.

There are some variants within the IBS family. 
For instance, survival continuous ranked probability score (survival-CRPS)~\cite{avati2020countdown}, also called integrated mean absolute error, simply omits the IPCW weights from the calculation, and integrated binomial log-likelihood uses the log-absolute error instead of squared error.
Due to their close resemblance, this paper will solely discuss IBS, as the same reasoning applies to other variants.

IBS is a negative scoring rule which means smaller score implies better performance.
It is also known to be a proper scoring rule~\cite{buja2005loss} if the censoring distribution is estimated correctly~\cite{kvamme2019brier, rindt2022survival}.
The relationship between IBS and MAE is subtle. 
Intuitively, IBS tries to minimize the summation of squared difference in the shaded areas ($\alpha + \beta$) as shown in Figure~\ref{fig:6_metrics} (c), while MAE tries to minimize the absolute difference between the two areas ($|\alpha - \beta |$). Motivate by this, we can conclude that:
\begin{itemize}
\itemsep0em
    \item minimizing IBS can lead to minimizing MAE;
    \item minimizing MAE does not necessarily lead to minimizing IBS.
\end{itemize}
One of the disadvantages of BS and IBS is that they are difficult to interpret, \ie they do not correspond to the expected time error nor to the ranking of the patient's time to event.
It might be useful when clinicians need to make decisions concerning time-specific probabilities, \ie conservative treatment if a 5-year survival probability is greater than 80\%.  

A further issue with the IBS with IPCW weights is that it is dominated by the few uncensored subjects (especially those who experience events at a late period) if the censoring rate is high, which implies a high variance. 
In other words, if only accurate prediction for those uncensored subjects were made at a late time (and an imprecise prediction for others), the overall IBS score will still tend to be good, and {\em vice versa}.
An extreme case happened when the dataset had some administrative censoring\footnote{Administrative censoring refers to the censorship that occurs when the study observation period ends.} subjects, the weights for those subjects cannot transfer to later subjects because they are the last observed time in the datasets; see also examples in~\cite{kvamme2019brier}.



\subsection{Log-likelihood}
The log-likelihood (LL) score measures the logarithmic values of the PDF function at the moment of the event (see Figure~\ref{fig:6_metrics} (d)). For censored patients, LL assesses the survival probability at the censoring time, and it is a positive scoring rule. Therefore, the greater the PDF intensity or survival probability, the higher the performance. Given the assumption of independent and noninformative censorship:
\begin{equation*}
\begin{aligned}
    \mathcal{R}_{LL} (\SurvSub{m}{\cdot}{\Bfx{i}}, t_i, \delta_i) = \delta_i \log f_{\text{m}}(t_i\mid \Bfx{i}) + (1 - \delta_i) \log \SurvSub{m}{t}{\Bfx{i}} \ .
\end{aligned}
\end{equation*}
LL has been widely used as the loss function to train an ISD model (see for example \citet{lee2018deephit, miscouridou2018deep}). \citet{rindt2022survival} proved that LL is a proper scoring rule. 
The following subsections prove that maximizing the LL will lead to minimizing the MAE-hinge.

Despite all LL's benefits, we discourage its usage as an evaluation metric as it does not have a boundary -- so it is not clear what value means we have a good model. 
Furthermore, considering the nature of the PDF function and probability mass function (PMF), 
we cannot use LL to compare 
(1) continuous and discrete models; 
(2) two discrete models with different bin boundaries;
nor (3) models of the same type trained with different datasets/time ranges.

\subsubsection{Minimize uncensored LL will lead to Minimize MAE-uncensored}
The expectation of the predicted survival probability is presented in Equation~\ref{eq:mean_survival_time} (here for simplicity, we omit the condition on $\Bfx{i}$).
For uncensored data, with maximizing the likelihood at the event time $f(e_i) \rightarrow \infty$, the expectation becomes:
\begin{equation*}
\begin{aligned}
    \mu_i &= \lim_{f(e_i) \rightarrow \infty} \int_0^{\infty} S(t) \ dt = \lim_{f(e_i) \rightarrow \infty} \int_0^{\infty} t \ f(t) \ dt \\
    &= \lim_{f(e_i) \rightarrow \infty} \left( \int_0^{e_i -\Delta t} t \ f(t) \ dt  + \int_{e_i -\Delta t}^{{e_i +\Delta t}} t \ f(t) \ dt + \int_{e_i +\Delta t}^\infty t \ f(t) \ dt\right) \\
    &= e_i + \lim_{f(e_i) \rightarrow \infty} \left( \int_0^{e_i -\Delta t} t \ f(t) \ dt  + \int_{e_i +\Delta t}^\infty t \ f(t) \ dt\right) \ ,
\end{aligned}
\end{equation*}
where the first term is the event time $e_i$ when $\Delta t \rightarrow \infty$, and the second term is close to zero ($f(t) \rightarrow 0 \ \text{for} \ t \neq e_i$). So the L1-loss $|\mu_i - e_i |$ will be close to 0.

\subsubsection{Minimize censored LL will lead to Minimize MAE-hinge}
For censored data, the MLE would be $S(c_i) \rightarrow 1$ ($c_i$ is the censored time), the expectation becomes:
\begin{equation*}
\begin{aligned}
    \mu_i &= \lim_{f(e_i) \rightarrow \infty} \int_0^{\infty} S(t) \ dt = \lim_{f(e_i) \rightarrow \infty} \left( \int_0^{c_i} S(t) \ dt  + \int_{c_i}^\infty S(t) \ dt\right) \\
    &= c_i + \lim_{f(e_i) \rightarrow \infty}  \int_{c_i}^\infty S(t) \ dt \ ,
\end{aligned}
\end{equation*}
where the L1-hinge loss max equals to $\max \{(c_i - \mu_i), 0\}=\max \{-\lim_{f(e_i) \rightarrow \infty}  \int_{c_i}^\infty S(t) \ dt, 0\}=0$.

\subsection{Hosmer-Lemeshow Calibration}
Hosmer-Lemeshow calibration (1-calibration)~\cite{hosmer1980goodness} is a statistical test to evaluate the calibration ability of the risk predictions at a specific time.  
Similar to C-index, we can substitute risk prediction with survival probability prediction to analyze the performance of ISD models.

To calculate 1-calibration at the specific time $t^*$, we first sort the predicted probabilities at this time for all patients and group them into $K$ bins, $B_1, \ldots, B_K$.
Within each bin, we calculate the expected number of events using the predicted probabilities, $\mathbb{E}_{\Bfx{i} \sim B_k} [\SurvSub{m}{t^*}{\Bfx{i}}]$, and compare it to the observed event rate. 
Finally, we use the Hosmer-Lemeshow test to assess if the expected and observed event rates are statistically similar.
Figure~\ref{fig:6_metrics} (e) demonstrates this process with four uncensored patients and two groups. 
In Group 1, the observed event rate is 0.5, as patient A was alive at t=2, while patient B died at the same time point. 
The expected number of events in this group is calculated as the average of 0.58 and 0.51, based on the probabilities observed at the intersections. 
Using the same reasoning, the observed event rate is $1$ and the expected number of events is $\frac{0.35 + 0.04}{2}$ in Group 2.
To handle censored subjects, we can use the KM estimator to approximate the observed statistics \cite{d2003evaluation}.

1-calibration can aid clinicians in making group-level decisions, such as arranging medical resources in response to COVID-19 lockdown restrictions being lifted.
For example, if a 1-calibrated (on 100-th day) model forecasts that the expected number of patients with severe symptoms in 100 days is 10,000, then we should arrange the among of ICU beds correspondingly because the expected number and observed numbers are statistically similar.

BS can be decomposed into a 1-calibration part~\cite{degroot1983comparison}. 
Let's first discretize the probability $S(t\mid \Bfx{i})$ by assuming there are $K$ distinct values for the probability.  
Therefore, everyone's predicted probability at time $t^*$ can be rounded to one of $\{p_k\}_{k=1}^K$. 
Let's $n_k(t^*)$ be the number of people that has the same probability $p_k$ at time $t^*$, and $\lambda_k(t^*)$ be the proportion of people in $n_k$ who have event happened before $t^*$, the BS for a set of uncensored instances can be represented as 
\begin{equation*}
    \mathcal{R}_{\text{BS}} (t^*, (\SurvSub{m}{\cdot}{\Bfx{i}}, t_i, \delta_i)) 
    = \frac{1}{N} \sum_{k=1}^{10} n_k(t^*) (\lambda_k(t^*) - p_k)^2 + \frac{1}{N} \sum_{k=1}^{10}n_k(t^*) \lambda_k(t^*) (1 - \lambda_k(t^*)) \ .
\end{equation*}
where the first term is equivalent to 1-calibration at the target time, and the second term is called refinement or sharpness at the target time. 
Some people may argue that there is no compelling demand to use 1-calibration if BS or IBS are used as the evaluation metrics.
However, as we can see from the decomposition, it is trivial to conclude that BS can be large while the 1-calibration term can be small. 
Therefore, if a model has a bad (large) BS or IBS score, it doesn't necessarily mean that it doesn't 1-calibrated.

\subsection{Distribution Calibration}

Distribution calibration (D-calibration)~\cite{haider2020effective} is also a statistical test to evaluate the calibration ability of the entire ISD prediction.
As to the notation, for any probability interval $[a, b] \subset [0, 1]$, let
\begin{equation*}
    \mathcal{D}_{\text{m}}(a, b) = \{ [\Bfx{i}, t_i, \delta_i=1] \in \mathcal{D} \mid \SurvSub{m}{t_i}{\Bfx{i}} \in [a, b]\} \ ,
\end{equation*}
be the subset of the subjects in the dataset $\mathcal{D}$ whose predicted probability at its event time, $\SurvSub{m}{t_i}{\Bfx{i}}$, is in the interval $[a, b]$. 
The model is D-calibrated if the proportion of patients $\frac{|D_{m}(a, b)|}{|D|}$ is statistically similar to the proportion $b - a$. Models can be trained to be D-calibrated using a differentiable relaxation \citep{goldstein2020x}.
\citet{haider2020effective} advises using equal-sized, mutually exclusive intervals with Pearson's $\chi^2$ test to examine if the proportion of patients in each bin is uniformly distributed.
Figure~\ref{fig:6_metrics} (f) provides a visual illustration of a D-calibrated model, where the predicted time-to-event probability is equally distributed across two intervals.
To incorporate censored patients, we can ``split'' each censored patient uniformly among the subsequent probability intervals after the time-to-censor-probability~\cite{haider2020effective}.

As D-calibration is a goodness-of-fit test that involves the entire distribution rather than a single time point, clinicians can use it to make group-level decisions with more flexibility (e.g., estimate the number of ICU beds available in 10 days for a group of patients recruited at various periods),
see motivation in~\citet{kumar2022learning}.

D-calibration and MAE measure different aspects of the performance, meaning they can give different orderings of models. 
For example, a KM model, despite being a perfectly D-calibrated model, will predict the same time for everyone, which can produce a terrible MAE score. 
Contrarily, a model with zero MAE score can have $\SurvSub{m}{t_i}{\Bfx{i}} = 0.5$ for all the subjects, resulting in one of the probability intervals containing all the patients for D-calibration ($p$-value = 0). This example is illustrated in Figure~\ref{fig:max_vs_dcal}.
Below we provide a theoretical comparison between these two metrics.

\begin{figure}[t]
    \centering
    \includegraphics[width=0.6\columnwidth]{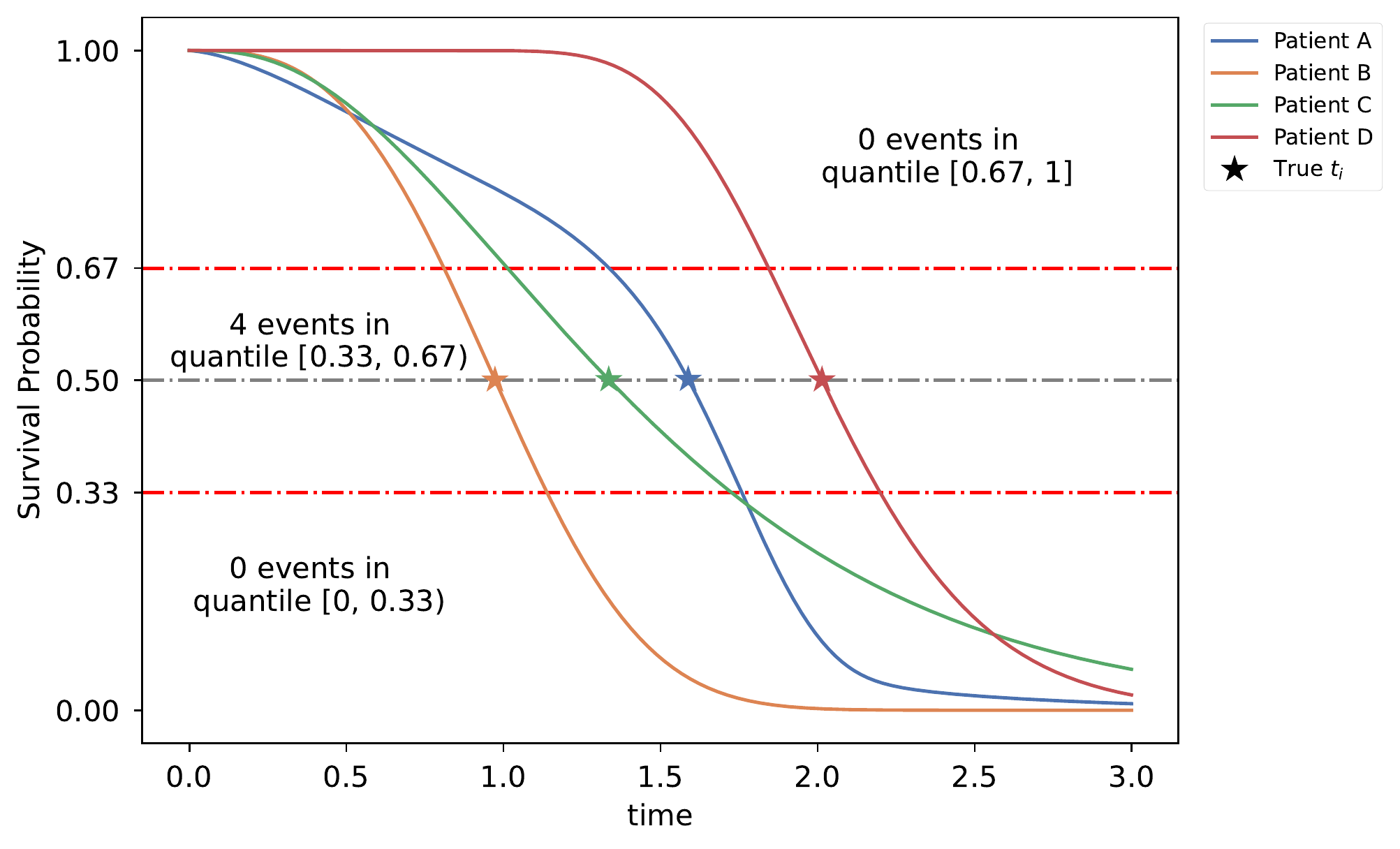}
    \caption{A toy example with four uncensored subjects, with best-possible MAE (MAE = 0) while not D-calibrated. Note that the true event times (position of stars) are overlapped with predicted median survival times.}
    \label{fig:max_vs_dcal}
\end{figure}

\begin{proposition}
    To demonstrate that MAE is fundamentally distinct from D-calibration, we show that it is possible for
\begin{itemize}
\itemsep0em
    \item a model to have small MAE (in fact, 0), but not be D-calibrated; and
    \item a model to be perfectly D-calibrated, but have arbitrarily large MAE (relative to the entire range of times).
\end{itemize}
\end{proposition}
\begin{proof}
    
In the following section, we prove the above proposition using uncensored datasets $\mathcal{D} = \{  (\Bfx{i}, t_i, \delta_i = 1) \}_{i=1}^N$ for the sake of simplicity.

\subsubsection{Small MAE but Poor D-calibration}

Assume survival curve $\SurvSub{m}{\cdot}{\Bfx{i}}$ for subject $\Bfx{i}$ is a logistic function, centered at the correct time-of-death $t_i$. 
Here, the median survival time of the estimated ISD is $t_i$, and the MAE error is zero (best MAE score possible) for all the subjects. 
However, as $\SurvSub{m}{t_i}{\Bfx{i}} = 0.5$ for all the subjects (because the logistic function is censored at $t_i$), therefore, one of the probability intervals must contain all the patients, leading to non-D-calibration.

\subsubsection{D-calibrated but Large MAE}

As the MAE has a physical interpretation, which is obviously related to the time ``units'' (\ie minutes, days, or years), we will normalize the times by dividing it by the largest value of the event times, so the largest value is 1 ($t_i \leq 1$ for all $i$).  
For any constant $R > 1$, we will provide a model that is D-calibrated (using 2 discretized bins), but the MAE is $\geq$ R. 
This hypothetical model produces linear survival curves.
Let assume that, for each subject $x_i$ (with time-of-death $e_i$), the curve starts from $(0,1)$ and ends at $(\frac{t_i}{1 - \kappa_i}, 0)$, and it descends linearly through $( t_i , \kappa_i )$ and crosses 0.5 at $(\frac{t_i}{2 (1 - \kappa_i)} , 0.5)$ (see Figure~\ref{fig:linear_curve_demo}).

\begin{figure}[ht]
    \centering
    \includegraphics[width=0.5\columnwidth]{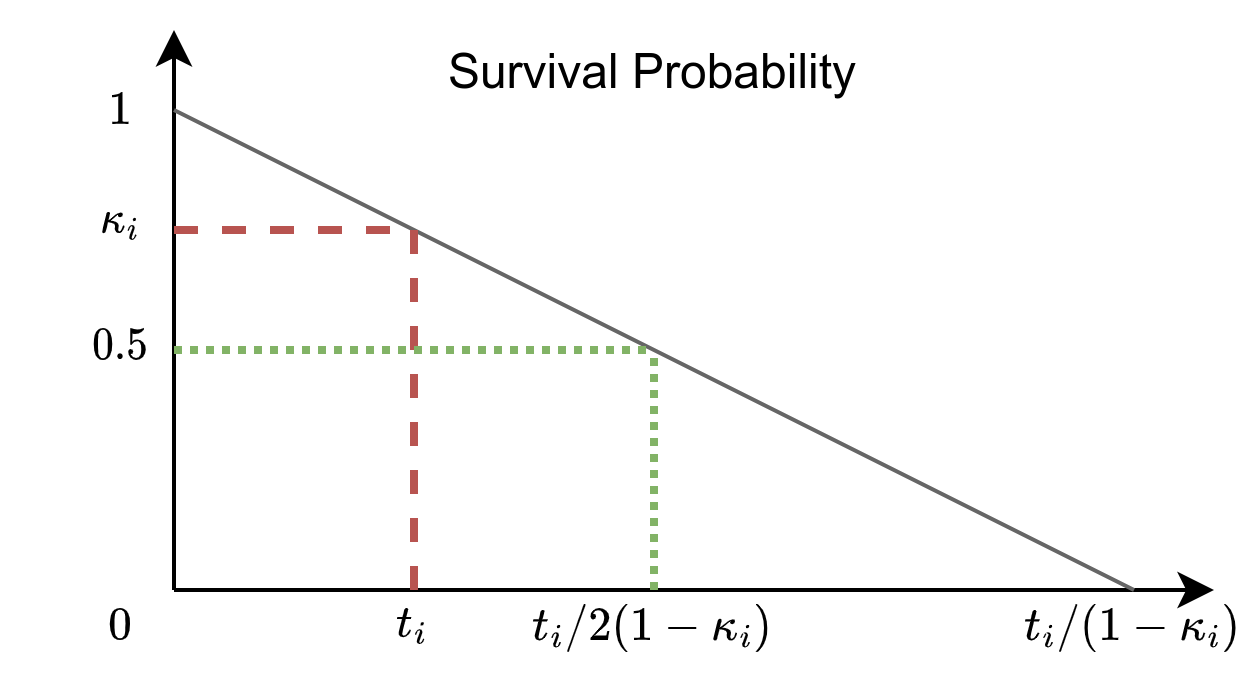}
    \caption{An toy example of a linear survival curve which starts at $(0, 1)$ and crosses $( t_i, \kappa_i )$.}
    \label{fig:linear_curve_demo}
\end{figure}

For the first half subjects, define $\kappa_i$  (the probability associated with the time of death $t_i$ ) to be $\kappa_i   =  1 - \frac{t_i}{4R + 2t_i}$ and for the other half subjects, set $\kappa_i = 0.25$. 

To show this model is D-calibrated with 2 bins ($|B| = 2$), observe that $t_i \leq 1 \leq R$ implies $\frac{t_i}{4R + 2t_i} \leq \frac{t_i}{4R} \leq \frac{1}{4}$, which means $\kappa_i = 1 - \frac{t_i}{4R + 2t_i} \geq 1 - \frac{1}{4} = \frac{3}{4} > \frac{1}{2}$. Hence, $\Surv{t_i}{\Bfx{i}} > \frac{1}{2}$ for each individual of the first half of the subjects, and $1-\frac{3}{4} = \frac{1}{4} < \frac{1}{2}$ for the second half.

Now consider the MAE: For the first half of the subjects, the median survival time for each individual is
\begin{equation*}
      \text{median}( \Surv{\cdot}{\Bfx{i}} ) =   \frac{t_i}{2 \cdot (1-\kappa_i)} =   \frac{t_i}{2} \frac{1}{1-\kappa_i} = \frac{t_i}{2} \frac{4R+2t_i}{t_i} =  2R+t_i \ .
\end{equation*}

This means the MAE for each of the patients is $|2R + t_i - t_i| = 2R$.

Hence, the overall MAE, over the entire set of $N$ patients, is
\begin{equation*}
\begin{aligned}
    \frac{1}{N}  \sum_{i=1}^N |   \text{median}( \Surv{\cdot}{\Bfx{i}} ) - t_i | 
    \geq \frac{1}{N}  \sum_{i=1}^N |  \text{median}( \Surv{\cdot}{\Bfx{i}})   - t_i | 
    = \frac{1}{n}              (\frac{n}{2}  \cdot  2R ) 
    =   R \ .
\end{aligned}
\end{equation*}
Then we complete the proof.
\end{proof}

\textbf{NOTES:}

(1) It is trivial to extend this proof to deal with $|B|=5$ or 10 different probability intervals (rather than the $|B|=2$ covered above):
For 1 bin, use $\kappa_i = 1 - \frac{t_i}{ 2|B| \cdot R + 2t_i}$.  For the other bins, we can use $\kappa _i$ in the 2nd, 3rd, $\dots$, up to B-th interval.
So for 10 bins, these would be $\kappa_i$ = 0.15,  then = 0.25, 0.35, …, 0.85, 0.95 – and everything will still follow the proof.

(2) If there are censored individuals, it is reasonable to have models that predict event times that extend beyond the final recorded time.  But if only uncensored instances, one could argue we should only consider times within the range of the training instances – in our case, with the time-of-events $\in [0,1]$. Of course, if this is the case, it is easy to have an MAE score bounded by 0.5 by choosing the mid-point time for each person. With this in mind, given any uncensored dataset, we can produce a model that (1) is D-calibrated, and (2) has MAE score $\geq \frac{1}{2}$.

\section{Properties of Pseudo-observation}
\label{appendix:pseudo-obs_property}
In this section, we prove some relevant conjectures about pseudo-observation~\cite{andersen2003generalised, andersen2010pseudo} to help justify the MAE-PO as an evaluation metric.

\subsection{Pseudo-observation Values Can be Treated as i.i.d. Random Variable.}
\label{appendix:pseudo-obs_iid}
\citet{graw2009pseudo} proved that the jackknife pseudo-observation values of competing risks could be expressed as a first Gateaux derivative of the Aalen-Johansen estimator~\cite{aalen1978nonparametric} and the cumulative incidence function, if $N \rightarrow \infty$.
Therefore, the pseudo-observation probabilities of competing risks can be approximated by independent and identically distributed variables. 
It is trivial to conclude that,
in the absence of a competing risk, the estimated cumulative incidence function will reduce to the cumulative density function (and the Aalen-Johansen estimator will reduce to KM estimator); hence, the property of independent and identically distributed pseudo-observations also holds for the survival function and the mean survival times.

\subsection{Authenticity of Pseudo-observation}
\label{appendix:pseudo-obs_authenticity}

In this section, we prove that the pseudo-observation value \cite{andersen2003generalised} of a censored instance using the Kaplan Meier (KM) estimator \cite{kaplan1958nonparametric} is always larger or equal to its censoring time. This property has been empirically demonstrated using synthetic examples in~\citet{andersen2010pseudo} but we will give a theoretical proof for the first time.
In the following, we represent the KM estimator as a stepwise function (see Equation~\ref{eq:unbiased_km} and~\ref{eq:biased_km}). Readers can use the linear interpolation version of KM without loss of generality. Furthermore, the Nelson-Aalen (NA) estimator~\cite{nelson1972theory, aalen1978nonparametric} has already been proven to be asymptotically equivalent to KM estimator~\cite{fleming2011counting}. Therefore, the following theoretical conclusions also hold if we use the NA estimator to predict survival curves.

First, let's recall the definition of pseudo-observation. For a censored instance $\delta_i=0$, its pseudo-observation can be expressed as:
\begin{equation}
\label{eq:pseudo_obs}
    e_{\text{pseudo-obs}} (i) = N \times \E_t [ \SurvKM{t}{\Data}] - (N-1) \times \E_t [\SurvKM{t}{\Data^{-i}}] \ ,
\end{equation}
where $\SurvKM{t}{\Data}$ is called the unbiased survival distribution estimator over the whole population using KM, while the $\SurvKM{t}{\Data^{-i}})$ is called the biased estimator over all the data except \textit{i}-th censored instance. Given these two estimations, we have:

\begin{lemma}
\label{lemma:unbiased_vs_biased}
Given an instance censored at $c_i$ with $\delta_i=0$, the unbiased population-based survival distribution estimator is always greater than or equal to the biased population-based estimators. Formally, 
\begin{equation*}
    \SurvKM{t}{\Data} \geq \SurvKM{t}{\Data^{-i}}, \quad \forall c_i,  \forall t \ .
\end{equation*}
\end{lemma}

\begin{proof}
The unbiased KM estimation for the whole dataset can be represented as:
\begin{equation}
\label{eq:unbiased_km}
    \SurvKM{t}{\Data} \ = \prod_{k:\: 0 \, \leq \, t_k < t} \frac{n_k - d_k}{n_k} \ ,
\end{equation}
with $t_k \in \mathbb{R}^+$ denotes a time when at least one event happened, $d_k \in \mathbb{Z}^+$ denotes the number of events that happened at time $t_k$, and $n_k \in \mathbb{Z}^+$ is the number of individuals known to be at-risk (have not yet had an event or been censored) up to time $t_k$. Then, with the same notation, the biased KM estimation for the leave-\textit{i}-out population can be represented as: 
\begin{equation}
\label{eq:biased_km}
    \SurvKM{t}{\Data^{-i}} \ = \prod_{k: \: 0 \, \leq \, t_k < t_j} \frac{n_k - 1 - d_k}{n_k - 1}  \times \prod_{k: \: t_j \, \leq \, t_k < t} \frac{n_k - d_k}{n_k} \ ,
\end{equation}
with $t_j \in \mathbb{R}^+$ denotes the next time after the censored time $c_m$ when at least one event happened. By definition, $t_j$ and $c_i$ will have the relationship $t_{j-1} < c_i \leq t_j$. Both the nominator and denominator decrease by 1 before $t_j$ due to this censored instance being eliminated from the at-risk population before $t_j$. Therefore, for any natural number $d_k$ given $t_k$, we have
\begin{flalign*}
    \frac{n_k - d_k}{n_k} \geq  \frac{n_k - 1 - d_k}{n_k - 1}, \quad \forall t_k \ 
    \Longrightarrow \ \prod_{k: \: 0 \, \leq \, t_k < t_j} \frac{n_k - d_k}{n_k} \geq  \prod_{k: \: 0 \, \leq \, t_k < t_j} \frac{n_k - 1 - d_k}{n_k - 1} \
    \Longrightarrow \ \SurvKM{t}{\Data} \geq  \SurvKM{t}{\Data^{-i}} \ .
\end{flalign*}
Then the proof is complete. We also present a toy example with visual illustration in Figure~\ref{fig:unbias_biase_visual} to demonstrate this Lemma. As we can see, the leave-$C_3$-out biased estimator is always lower or equal to the unbiased estimator at all times. 
\end{proof}

\begin{figure}[ht]
    \centering
    \includegraphics[width=0.7\columnwidth]{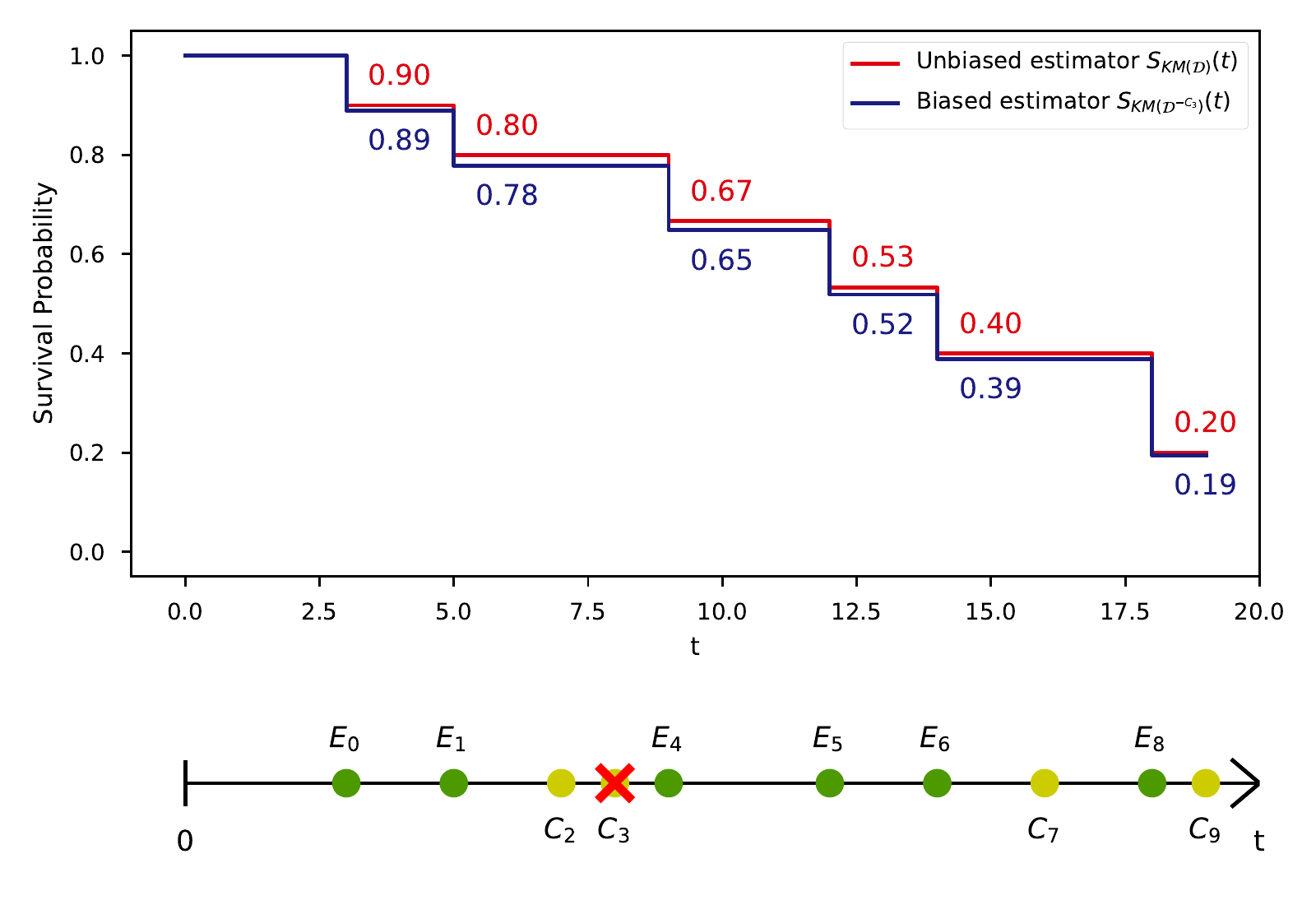}
    \caption{A visual comparison between the unbiased estimator and biased estimator using the Kaplan-Meier method on a toy example. The bottom dots in the $x$-axis shows the order of the event/censor where E (green dots) and C (yellow dots) represent the event and censor respectively. The biased estimator is a KM estimation for the leave-$C_3$-out population.}
    \label{fig:unbias_biase_visual}
\end{figure}

\begin{lemma}
\label{lemma:one_censor_bound}
For a survival dataset with only one censored instance $\{\Bfx{m}, t_m=c_m, \delta_m=0\}$ (also $\delta_i=1$ unless $i=m$), the pseudo-observation value (calculated using KM estimators) for this censored instance is lower bound by its censoring time. 
\begin{equation*}
    e_{\text{pseudo-obs}} (m) \, = \, N \times \E_t [ \SurvKM{t}{\Data}] - (N-1) \times \E_t [\SurvKM{t}{\Data^{-m}}] \, \geq \, c_m \ .
\end{equation*}
\end{lemma}

\begin{proof}
We start by rearranging the expression in Equation~\ref{eq:pseudo_obs}.
\begin{equation*}
\begin{aligned}
    e_{\text{pseudo-obs}} (m)
    &= \E_t [ \SurvKM{t}{\Data}] + (N - 1)\left[\E_t [ \SurvKM{t}{\Data}] - \E_t [\SurvKM{t}{\Data^{-m}}]\right] \\
    &= \int_{0}^{t_j} \SurvKM{t}{\Data} \, dt + (N - 1)\int_{0}^{t_j} \left(\SurvKM{t}{\Data} - \SurvKM{t}{\Data^{-m}} \right) \, dt \\
    & \qquad + \int_{t_j}^{\infty} \SurvKM{t}{\Data} \, dt + (N - 1)\int_{t_j}^{\infty} \left(\SurvKM{t}{\Data} - \SurvKM{t}{\Data^{-m}} \right) \, dt \\
    &\geq \int_{0}^{t_j} \SurvKM{t}{\Data} \, dt + (N - 1)\int_{0}^{t_j} \left(\SurvKM{t}{\Data} - \SurvKM{t}{\Data^{-m}} \right) \, dt \ .
\end{aligned}
\end{equation*}
The first equality is due to rearranging the factors. The second equality holds by substituting the KM estimators by Equation~\ref{eq:unbiased_km} and~\ref{eq:biased_km}. The third equality separates the range for the integral. And the inequality is because of Lemma~\ref{lemma:unbiased_vs_biased}. Therefore, we can complete the proof as long as we can prove the following inequality:
\begin{equation}
\label{eq:substitute_theorem}
    \int_{0}^{t_j} \SurvKM{t}{\Data} \, dt + (N - 1)\int_{0}^{t_j} \left(\SurvKM{t}{\Data} - \SurvKM{t}{\Data^{-m}} \right) \, dt \; \geq \; c_m \ .
\end{equation}
Or alternatively, by substituting the KM estimators by Equation~\ref{eq:unbiased_km} and~\ref{eq:biased_km}, 
\begin{equation*}
    \int_{0}^{t_j} \prod_{k:\: 0 \, \leq \, t_k < t} \frac{n_k - d_k}{n_k} \, dt + (N - 1)\int_{0}^{t_j} \left( \prod_{k:\: 0 \, \leq \, t_k < t} \frac{n_k - d_k}{n_k} - \prod_{k: \: 0 \, \leq \, t_k < t_j} \frac{n_k - 1 - d_k}{n_k - 1} \right) \, dt \; \geq \; c_m \ .
\end{equation*}
For a survival dataset that contains \emph{only one censored instance}, and $t_{j-1} < c_m \leq t_j$. We have 
\begin{equation*}
    n_{k-1} - d_{k-1} = n_{k}, \quad \forall t_k \; \text{except} \; k=j \ .
\end{equation*}
Therefore, using the above equation, the first term in Equation~\ref{eq:substitute_theorem} can be expressed as:
\begin{equation}
\label{eq:km_term_one_censor}
\begin{aligned}
    \int_{0}^{t_j} \SurvKM{t}{\Data} &\, dt =  \int_{0}^{t_j} \prod_{k:\: 0 \, \leq \, t_k < t} \frac{n_k - d_k}{n_k} \, dt \\
    &= \int_{0}^{t_1} \frac{n_0 - d_0}{n_0} \, dt + \int_{t_1}^{t_2} \frac{n_0 - d_0}{n_0} \frac{n_1 - d_1}{n_1} \, dt + \cdots + \int_{t_{j-1}}^{t_j} \frac{n_0 - d_0}{n_0} \cdots \frac{n_{j-2} - d_{j-2}}{n_{j-2}} \frac{n_{j-1} - d_{j-1}}{n_{j-1}} \, dt \\
    &= \int_{0}^{t_1} \frac{n_0 - d_0}{n_0} \, dt + \int_{t_1}^{t_2} \frac{n_1 - d_1}{n_0} \, dt + \cdots + \int_{t_{j-1}}^{t_j} \frac{n_{j-1} - d_{j-1}}{n_0} \, dt \ .
\end{aligned}
\end{equation}
Similarly, the second term can be expressed as:
\begin{equation}
\label{eq:deduction_term_one_censor}
\begin{aligned}
    (N - 1)\int_{0}^{t_j} &\left( \SurvKM{t}{\Data} - \SurvKM{t}{\Data^{-m}} \right) \, dt \\
    &= (n - 1) \int_{0}^{t_j} (\prod_{k:\: 0 \, \leq \, t_k < t} \frac{n_k - d_k}{n_k} - \prod_{k: \: 0 \, \leq \, t_k < t_j} \frac{n_k - 1 - d_k}{n_k - 1}) \, dt \\
    &= (n-1) \left(\int_{0}^{t_1} (\frac{n_0 - d_0}{n_0} - \frac{n_0 - 1 - d_0}{n_0 - 1})\, dt + \int_{t_1}^{t_2} (\frac{n_1 - d_1}{n_0} - \frac{n_1 - 1 - d_1}{n_0 - 1}) \, dt \right. \\
    & \qquad \left. + \cdots + \int_{t_{j-1}}^{t_j} (\frac{n_{j-1} - d_{j-1}}{n_0} - \frac{n_{j-1} - 1 - d_{j-1}}{n_0 - 1})\, dt \right) \\
    &= (n-1) \left(\int_{0}^{t_1} \frac{d_0}{n_0(n_0 -1)}\, dt + \int_{t_1}^{t_2} \frac{n_0 -n_1 + d_1}{n_0(n_0 - 1)} \, dt + \cdots + \int_{t_{j-1}}^{t_j} \frac{n_0 - n_{j-1} + d_{j-1}}{n_0(n_0 - 1)}\, dt \right) \\
    &= \int_{0}^{t_1} \frac{d_0}{n_0}\, dt + \int_{t_1}^{t_2} \frac{n_0 -n_1 + d_1}{n_0} \, dt + \cdots + \int_{t_{j-1}}^{t_j} \frac{n_0 - n_{j-1} + d_{j-1}}{n_0}\, dt \\
    &= \int_{0}^{t_1} (1 - \frac{n_0 - d_0}{n_0})\, dt + \int_{t_1}^{t_2} (1 - \frac{n_1 - d_1}{n_0}) \, dt + \cdots + \int_{t_{j-1}}^{t_j} (1 - \frac{n_{j-1} - d_{j-1}}{n_0})\, dt \ .
\end{aligned}
\end{equation}
The third equality holds because $n_0 = N$ for the dataset with only one censored instance. We can easily observe that each term in Equation~\ref{eq:km_term_one_censor} just complements the corresponding term in Equation~\ref{eq:deduction_term_one_censor}. Therefore, we can have
\begin{equation*}
        \int_{0}^{t_j} \SurvKM{t}{\Data} \, dt + (N - 1)\int_{0}^{t_j} \left(\SurvKM{t}{\Data} - \SurvKM{t}{\Data^{-m}} \right) \, dt
        = \int_{0}^{t_j} 1\, dt = t_j \; \geq \; c_m \ ,
\end{equation*}
then we complete the proof.
\end{proof}

\begin{theorem}
\label{theorem:authenticity}
For a survival dataset with arbitrary numbers of censored and event instances, the pseudo-observation value (calculated using KM estimators) for any censored instance is lower bound by its censoring time. 
\begin{equation*}
    e_{\text{pseudo-obs}} (i) = N \times \E_t [ \SurvKM{t}{\Data}] - (N-1) \times \E_t [\SurvKM{t}{\Data^{-i}}] \; \geq \; c_i \ .
\end{equation*}
\end{theorem}

\begin{proof}
We can still follow the intuition and steps in the proof for Lemma~\ref{lemma:one_censor_bound}. That means, as long as we can prove the correctness of Equation~\ref{eq:substitute_theorem} in this circumstance, this theorem can be proved. 

In the case of an unlimited number of censor instances, the number of survival instances at the previous time must be large or equal to the at-risk instances at the next time, which means:
\begin{flalign}
    \label{eq:inequality_1_any_censor}
    && \frac{n_{k-1} - d_{k-1}}{n_{k}} & \geq 1 , \quad \forall t_k \ , && \\
    \label{eq:inequality_2_any_censor}
    \Longrightarrow && \prod_{k: \: t_1 \, \leq \, t_k < t_j} \frac{n_{k-1} - d_{k-1}}{n_{k}} & \geq \prod_{k: \: t_1 \, \leq \, t_k < t_j} \frac{n_{k-1} - 1 - d_{k-1}}{n_{k} - 1} \geq 1 \ . &&  
\end{flalign}
Therefore, using the above equations, the first term in Equation~\ref{eq:substitute_theorem} can be expressed as:
\begin{equation}
\label{eq:km_term_any_censor}
\begin{aligned}
    \int_{0}^{t_j} \SurvKM{t}{\Data} &\, dt \\
    &= \int_{0}^{t_1} \frac{n_0 - d_0}{n_0} \, dt + \int_{t_1}^{t_2} \frac{n_0 - d_0}{n_0} \frac{n_1 - d_1}{n_1} \, dt + \cdots + \int_{t_{j-1}}^{t_j} \frac{n_0 - d_0}{n_0} \cdots \frac{n_{j-2} - d_{j-2}}{n_{j-2}}\frac{n_{j-1} - d_{j-1}}{n_{j-1}} \, dt \\
    &= \int_{0}^{t_1} \frac{n_0 - d_0}{n_0} \, dt + \int_{t_1}^{t_2} \frac{n_0 - d_0}{n_1} \frac{n_1 - d_1}{n_0} \, dt + \cdots + \int_{t_{j-1}}^{t_j} \frac{n_0 - d_0}{n_1} \cdots \frac{n_{j-2} - d_{j-2}}{n_{j-1}}\frac{n_{j-1} - d_{j-1}}{n_0} \, dt \\
    &\geq \int_{0}^{t_1} \frac{n_0 - d_0}{n_0} \, dt + \int_{t_1}^{t_2} \frac{n_1- d_1}{n_0} \, dt + \cdots + \int_{t_{j-1}}^{t_j} \frac{n_{j-1} - d_{j-1}}{n_0} \, dt \ .
\end{aligned}
\end{equation}
The second equality is simply by changing the position of denominators. The inequality holds because of Equation~\ref{eq:inequality_1_any_censor}. Similarly, the second term in Equation~\ref{eq:substitute_theorem} can be expressed as:
\begin{equation}
\label{eq:deduction_term_any_censor}
\begin{aligned}
    &(N - 1)\int_{0}^{t_j} \left(\SurvKM{t}{\Data} - \SurvKM{t}{\Data^{-i}} \right) \, dt\\
    &= (N-1) \left( \int_{0}^{t_1} (\frac{n_0 - d_0}{n_0} - \frac{n_0 - 1 - d_0}{n_0 - 1})\, dt + \int_{t_1}^{t_2} (\frac{n_0 - d_0}{n_1} \frac{n_1 - d_1}{n_0} - \frac{n_0 - 1 - d_0}{n_1 - 1} \frac{n_1 - 1 - d_1}{n_0 - 1}) \, dt \right. \\
    & \quad + \left. \cdots + \int_{t_{j-1}}^{t_j} (\frac{n_0 - d_0}{n_1} \cdots \frac{n_{j-2} - d_{j-2}}{n_{j-1}} \frac{n_{j-1} - d_{j-1}}{n_0} - \frac{n_0 - 1 - d_0}{n_1 - 1} \cdots \frac{n_{j-2} - 1 - d_{j-2}}{n_{j-1} - 1} \frac{n_{j-1} - 1 - d_{j-1}}{n_0 - 1} )\, dt \right)\\
    &\geq (N-1) \left(\int_{0}^{t_1} \frac{d_0}{n_0(n_0 -1)}\, dt + \int_{t_1}^{t_2} \frac{n_0 - n_1 + d_1}{n_0(n_0 - 1)} \, dt + \cdots + \int_{t_{j-1}}^{t_j} \frac{n_0 - n_{j-1} + d_{j-1}}{n_0(n_0 - 1)}\, dt \right) \\
    &\geq  \int_{0}^{t_1} (1 - \frac{n_0 - d_0}{n_0})\, dt + \int_{t_1}^{t_2} (1 - \frac{n_1 - d_1}{n_0}) \, dt + \cdots + \int_{t_{j-1}}^{t_j} (1 - \frac{n_{j-1} - d_{j-1}}{n_0})\, dt \ .
\end{aligned}
\end{equation}
The first equality is also by changing the position of denominators. The first inequality holds because of Equation~\ref{eq:inequality_2_any_censor}. The second inequality follows the rules of Equation~\ref{eq:deduction_term_one_censor}, as well as $n_1 \leq N$ for any time for the dataset with unlimited censored instances. As Lemma~\ref{lemma:one_censor_bound}, we can now combine the Equations~\ref{eq:km_term_any_censor} and \ref{eq:deduction_term_any_censor}:
\begin{equation*}
        \int_{0}^{t_j} \SurvKM{t}{\Data} \, dt + (N - 1)\int_{0}^{t_j} \left(\SurvKM{t}{\Data} - \SurvKM{t}{\Data^{-i}} \right) \, dt
        \geq \int_{0}^{t_j} 1\, dt = t_j \; \geq \; c_i \ ,
\end{equation*}
then we complete the proof.
\end{proof}

\subsection{Susceptible to Dataset Size}

We observe that the pseudo-observation value for a censored subject in the dataset is not ``invariant'' under duplication of subjects in the dataset, unlike other MAE-based metrics such as MAE-margin or MAE-IPCW-T.
For example, if we have a dataset $\Data$ with 100 subjects and one censored subject, $i$, we can calculate the pseudo-observation value for $i$ using two Kaplan-Meier (KM) estimations. 
However, if we duplicate every subject in $\Data$ and create a new dataset $\Data'$ with 200 subjects and two censored subjects (including $i$), the pseudo-observation values for $i$ in $\Data$ will not be the same as the PO for that one $i$ in $\Data$, despite the KM curves being identical in both datasets ($\SurvKM{t}{\Data} = \SurvKM{t}{\Data'}$ for all $t$). 
We called this property ``susceptible to dataset size''.

However, we view this property neither as an advantage nor a limitation. 
In one way, it is arguable that duplicating the subjects violates the i.i.d. assumption of the data, which makes two different datasets and therefore the surrogate event times for the same subjects should be different. 
In another way, if the KM curves represent the true survival distribution of the datasets, then the same subjects in the same survival distribution should have the same surrogate times, no matter the dataset size.

\subsection{Relationship between Pseudo-observation and Other MAE Metrics}
In this section, we investigate the properties of pseudo-observations by comparing them with other MAE metrics.

\begin{theorem}
For a survival dataset with only one censored instance $\{\Bfx{m}, t_m=c_m, \delta_m=0\}$ (also $\delta_i=1$ unless $i=m$), the pseudo-observation value and the margin ``best-guess'' value (both calculated using KM estimators) are the same, $e_{\text{pseudo-obs}} (m) = e_{\text{margin}} (m)$.
By their definition, it equals to:
\begin{equation*}
   N \times \E_t [ \SurvKM{t}{\Data}] - (N-1) \times \E_t [\SurvKM{t}{\Data^{-m}}] = c_m + \frac{\int_{c_m}^{\infty} \SurvKM{t}{\Data} dt}{\SurvKM{c_m}{\Data}} \ .
\end{equation*}
\end{theorem}

\begin{proof}
Recall the proof in Lemma~\ref{lemma:one_censor_bound}, we can have:
\begin{equation*}
\begin{aligned}
    e_{\text{pseudo-obs}} (m)
    &= \int_{0}^{t_j} \SurvKM{t}{\Data} \, dt + (n - 1)\int_{0}^{t_j} \left(\SurvKM{t}{\Data} - \SurvKM{t}{\Data^{-m}} \right) \, dt \\
    & \qquad + \int_{t_j}^{\infty} \SurvKM{t}{\Data} \, dt + (n - 1)\int_{t_j}^{\infty} \left(\SurvKM{t}{\Data} - \SurvKM{t}{\Data^{-m}} \right) \, dt \\
    &= t_j + \int_{t_j}^{\infty} \SurvKM{t}{\Data} \, dt + (n - 1)\int_{t_j}^{\infty} \left(\SurvKM{t}{\Data} - \SurvKM{t}{\Data^{-m}} \right) \, dt \ ,
\end{aligned}
\end{equation*}
in the circumstance of only one censored instance. The sum of the first two terms is equal to $t_j$ because of Equation~\ref{eq:km_term_one_censor} and~\ref{eq:deduction_term_one_censor}. Then, we can replace the unbiased and biased KM estimation by Equation~\ref{eq:unbiased_km} and~\ref{eq:biased_km} (also use the fact that $n_{k-1} - d_{k-1} = n_{k}$ for all $t_k$ except $k=j$). 
\begin{equation*}
\begin{aligned}
    e_{\text{pseudo-obs}} (m)
    &= t_j + \int_{t_j}^{\infty} \frac{n_{j-1} - d_{j-1}}{n_0} \prod_{k: \: t_j \, \leq \, t_k < t} \frac{n_k - d_k}{n_k} \, dt \\
    & \quad + (N - 1)\int_{t_j}^{\infty} ( \frac{n_{j-1} - d_{j-1}}{n_0} - \frac{n_{j-1} - 1 - d_{j-1}}{n_0 - 1}) \prod_{k: \: t_j \, \leq \, t_k < t} \frac{n_k - d_k}{n_k} \, dt \\
    &= t_j + \int_{t_j}^{\infty} ( \frac{n_{j-1} - d_{j-1}}{n_0} + \frac{n_0 - n_{j-1} + d_{j-1}}{n_0}) \prod_{k: \: t_j \, \leq \, t_k < t} \frac{n_k - d_k}{n_k} \, dt \\
    &= t_j + \int_{t_j}^{\infty} \prod_{k: \: t_j \, \leq \, t_k < t} \frac{n_k - d_k}{n_k} \, dt \ .
\end{aligned}
\end{equation*}
The second equality is from factorization with $N=n_0$. Similarly, we can get a derivation for the margin best-guess time:
\begin{equation*}
\begin{aligned}
    e_{\text{margin}} (m)
    &= c_m + \frac{\int_{c_m}^{\infty} \SurvKM{t}{\Data} \ dt}{\SurvKM{c_m}{\Data}} \\
    &= c_m + \frac{\int_{c_m}^{t_j} \SurvKM{c_m}{\Data} \ dt + \int_{t_j}^{\infty} \SurvKM{t}{\Data} \ dt}{\SurvKM{c_m}{\Data}} \\
    &= c_m + t_j - c_m + \frac{\int_{t_j}^{\infty} \SurvKM{c_m}{\Data} \prod_{k: \: t_j \, \leq \, t_k < t} \frac{n_k - d_k}{n_k} \ dt}{\SurvKM{c_m}{\Data}} \\
    &= t_j + \int_{t_j}^{\infty} \prod_{k: \: t_j \, \leq \, t_k < t} \frac{n_k - d_k}{n_k} \, dt \ .
\end{aligned}
\end{equation*}
Here we complete the proof by showing that the derivation of pseudo-observations and margin best-guess values is the same. Please note that this derivation for the margin best-guess time is not limited to this special dataset with only one censored instance.
\end{proof}

\begin{lemma}
For a survival dataset with arbitrary numbers of censored and event instances, the pseudo-observation value for any censored instance is always higher or equal to the margin best-guess value.
Formally:
\begin{equation*}
   N \times \E_t [ \SurvKM{t}{\Data}] - (N-1) \times \E_t [\SurvKM{t}{\Data^{-m}}] \geq c_m + \frac{\int_{c_m}^{\infty} \SurvKM{t}{\Data} \, dt}{\SurvKM{c_m}{\Data}} = t_j + \int_{t_j}^{\infty} \prod_{k: \: t_j \, \leq \, t_k < t} \frac{n_k - d_k}{n_k} \, dt \ .
\end{equation*}
\end{lemma}

\begin{proof}
In the case of arbitrary numbers of censored and event instances, the pseudo-observation value has the following relationship with the KM estimators:
\begin{equation*}
\begin{aligned}
    e_{\text{pseudo-obs}} (m) \geq t_j + \int_{t_j}^{\infty} \SurvKM{t}{\Data} \, dt + (N - 1)\int_{t_j}^{\infty} \left(\SurvKM{t}{\Data} - \SurvKM{t}{\Data^{-m}} \right) \, dt \ ,
\end{aligned}
\end{equation*}
The inequality is due to the derivation in Equation~\ref{eq:km_term_any_censor} and~\ref{eq:deduction_term_any_censor}. Following the idea in Theorem~\ref{theorem:authenticity}, it is trivial to prove that the sum of the second and third terms in the above equation is greater than or equal to $\int_{t_j}^{\infty} \prod_{k: \: t_j \, \leq \, t_k < t} \frac{n_k - d_k}{n_k} \, dt$. Then we complete the proof.
\end{proof}

\section{Experimental Details}
\label{appendix:exp_details}

\subsection{Datasets and Preprocessing}
\label{appendix:data_details}
In this section, we will describe how we preprocess the raw survival datasets.
\subsubsection{GBM}
GBM is retrieved from The Cancer Genome Atlas (TCGA) dataset~\cite{weinstein2013cancer}.
We only select patients diagnosed with glioblastoma multiforme cancer to build the GBM dataset. 
The data from TCGA can be found on \url{http://firebrowse.org/} or by the instruction in~\citet{haider2020effective}.
There are three features (radiation therapy, Karnofsky performance score, and ethnicity) containing the missing values.
We will use their median value to fill in the missing values.

\subsubsection{SUPPORT}
The Study to Understand Prognoses Preferences Outcomes and Risks of Treatment (SUPPORT) dataset~\cite{knaus1995support} comprises 8873 participants with the aim of examining survival outcomes and clinical decision-making for seriously ill hospitalized patients. 
The dataset consists of a proportion of missing values for a large proportion of features. 
The official website (\url{https://biostat.app.vumc.org/wiki/Main/SupportDesc}) for the SUPPORT dataset provides a guideline for imputing baseline physiologic features, we followed that procedure. 
For the rest features with missing values, we will also use the median value imputation.

\subsubsection{METABRIC}
The Molecular Taxonomy of Breast Cancer International Consortium (METABRIC)~\cite{curtis2012genomic} contains survival information for breast cancer patients.
The feature sets contain genetic and protein expression features.
The dataset can be downloaded from (\url{https://github.com/havakv/pycox}), and it does not have any missing values.

\subsubsection{MIMIC-IV}
The Medical Information Mart for Intensive Care (MIMIC)-IV~\cite{johnson2022mimic} dataset is an update to MIMIC-III, which provides critical care data from patients admitted to hospital and intensive care units (ICU).
We create two datasets using the MIMIC-IV database. 

\emph{MIMIC-IV (all-cause mortality)} contains patients that are alive at least 24 hours after being admitted to ICU.
Their date of death is derived from hospital records or state records, which means, the cause of mortality is not limited to the reason for ICU admission. 
We follow the instruction from~\cite{Han2022SurvivalMD} to make the dataset. 
However, there are two differences between our dataset and theirs.
The first is that their paper used MIMIC-IV v1.4 while we used the latest version, MIMIC-IV v2.0.
The second is if a patient got admitted to ICU multiple times, we will only include the last admission while they will consider each visit as a separate data.
The SQL code and python code that prepossesses the data from MIMIC-IV database is available in our GitHub repository. 

\emph{MIMIC-IV (hospital cause mortality)} contains patients that are alive at least 24 hours after admitting to the hospital. 
Their date of death is only derived from hospital records, so the direct cause of death will be the same as the cause of hospital admission.
Because time-series lab features are only available for patients admitted to ICU, we can only use the demographic and clinical features to describe each patient. 
The SQL code and python code that prepossesses the data from MIMIC-IV database is available in our GitHub repository.

\subsection{Model Implementation Details and Hyperparameter Choices}
\label{appendix:model_details}

In this section, we will describe the implementation of the models utilized in the performance comparison. Table~\ref{tab:model_comp} also provides a summary of the model comparison. 
Because the purpose of this work is to compare evaluation metrics, extensive hyperparameter searches for each model are unnecessary. 
Rather, we would like that the models display distinguishable performance.
Please do not view these results as a definitive evaluation of the robust performance of survival models.

\begin{table}[!t]
\centering
\caption{Comparison between the baseline time-to-event models. }
\label{tab:model_comp}
\begin{tabular}{ccccc}
\toprule
             & Survival Curves & Individual Prediction & Time-Prediction & Continuous \\ \midrule
LR           & {\color{blue}\xmark}       & {\color{red}\cmark}       & {\color{red}\cmark}       & N/A                \\
KM           & {\color{red}\cmark}       & {\color{blue}\xmark}       & {\color{blue}\xmark}       & {\color{blue}\xmark}             \\
CoxPH        & {\color{red}\cmark}$^\dagger$    & {\color{red}\cmark}     & {\color{blue}\xmark}       & {\color{red}\cmark}$^\ddagger$            \\
AFT          & {\color{red}\cmark}       & {\color{red}\cmark}       & {\color{red}\cmark}       & {\color{red}\cmark}             \\
RSF          & {\color{red}\cmark}       & {\color{red}\cmark}       & {\color{blue}\xmark}       & {\color{blue}\xmark}             \\
GBCM         & {\color{red}\cmark}       & {\color{red}\cmark}       & {\color{blue}\xmark}       & {\color{blue}\xmark}             \\
MTLR         & {\color{red}\cmark}       & {\color{red}\cmark}       & {\color{blue}\xmark}       & {\color{blue}\xmark}             \\
DeepHit      & {\color{red}\cmark}       & {\color{red}\cmark}       & {\color{blue}\xmark}       & {\color{blue}\xmark}             \\
SCA          & {\color{red}\cmark}       & {\color{red}\cmark}       & {\color{blue}\xmark}       & {\color{red}\cmark}             \\
S-MDN        & {\color{red}\cmark}       & {\color{red}\cmark}       & {\color{red}\cmark}       & {\color{red}\cmark}             \\
\bottomrule
\end{tabular}
\\
$^\dagger$ Naive CoxPH only predicts risk scores, whereas its Breslow extension allows the model to generate survival curves. \\
$^\ddagger$ Although CoxPH model assumes the risk is a time-invariant score in continuous time, the baseline hazard function is estimated through the discrete-time Breslow estimator. 
\end{table}

\emph{Linear regression (LR)} is a regressor model. We implemented an LR model using an Adam optimizer. LR is not a survival prediction model as it cannot handle censored subjects nor generate ISD and risk predictions. Instead, we only used the uncensored subjects in the training set to train the model. The model generated the estimated event time for the full test set, and we can perform the evaluation on the full test set. The model is implemented using \texttt{scikit-learn} packages.

\emph{Kaplan Meier~\cite{kaplan1958nonparametric}} is a non-parametric estimator to predict the survival distribution for a group of subjects. It is not a personalized prediction tool. We use the median survival time of the training set's population-level survival distribution as the predicted time for all the testing subjects and perform the evaluation. The model is implemented using \texttt{lifelines} packages.

\emph{CoxPH~\cite{cox1972regression} with Breslow estimator~\cite{breslow1975analysis}} is a semi-parametric model. It consists of a population-level baseline hazard function (non-parametric) and a partial hazard function (parametric). In the model implementation, the population-level baseline hazard function is estimated using Breslow method~\cite{breslow1975analysis}. And the partial hazard function is estimated by a linear function. The model is optimized using partial likelihood loss~\cite{cox1975partial} and Adam optimizer. The model is implemented in the code base attached. The early stop technique is applied to the model via validating on a separate validation set 

\emph{AFT~\cite{stute1993consistent}} with Weibull distribution is a parametric model with two estimated coefficients (a scale parameter and a shape parameter). We add a small l2 penalty to the loss for the model optimization. The method is implemented using \texttt{lifelines} packages.

\emph{GBM-C~\cite{hothorn2006survival}} is an ensemble method with component-wise least squares as the base learner. We use the 100 boosting stages with partial likelihood loss for optimization. The method is implemented using \texttt{scikit-survival} packages.

\emph{RSF~\cite{ishwaran2008random}} is also an ensemble estimator that fits a number of survival trees on bootstrapping datasets. We use 50 trees with 3 minimal samples per leaf to fit the model. The method is implemented using \texttt{scikit-survival} packages.

\emph{MTLR~\cite{yu2011learning}} is a discrete model that directly models the survival distribution for each individual. The number of discrete times is determined by the square root of numbers of uncensored patients, and use quantiles to divide those uncensored instances evenly into each time interval, as suggested in~\cite{jin2015using, haider2020effective}. The early stop technique is applied to the model via validating on a separate validation set 
The model is implemented in the code base attached.

\emph{DeepHit~\cite{lee2018deephit}} is also a discrete model. It models the probability density function of the event for each individual (and PDF can be used to calculate the survival distribution accordingly). The number of discrete times is determined by the square root of number of uncensored patients, just like MTLR. However, the time interval is uniformly split from time zero to the last observed time, as in the original paper~\cite{lee2018deephit}. Early stopping is also performed during the optimization. The model is implemented using \texttt{pycox} packages. 

\emph{SCA~\cite{chapfuwa2020survival}} models the covariates into a mixture-of-distributions latent space. And each component in the latent space will be used to stochastically predict/sample the survival distribution. We will use a three-hidden-layer structure with dimensions of [50, 50, 50]. 
We set the number of components to 25, 
kept the probability for weights equal to 0.8, and 
set the sample size to 200. 
Early stopping is also performed with at least 10000 epochs for guaranteed improvement. The model is implemented using the code base in \url{https://github.com/paidamoyo/survival_cluster_analysis}.

\emph{S-MDN~\cite{Han2022SurvivalMD}} uses Mixture Density Networks to model the survival distributions. The model architecture in the experiment has one hidden layer with a size of 15. The number of components is set to 15, and use residual as the initial type. The model is optimized via RMSprop optimizer with early stopping. The model is implemented using the code base in \url{https://github.com/XintianHan/Survival-MDN}.

 For further details, we refer to the code base attached.

\section{Complete Results}
\label{appendix:results_full}

\begin{figure}[!t]
    \centering
    \includegraphics[width=\columnwidth]{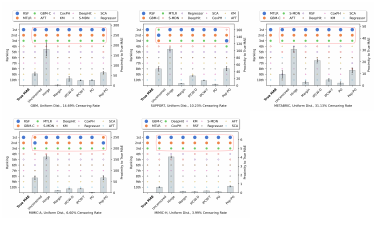}
    \caption{Evaluation metrics comparison on uniform censoring in terms of ranking accuracy (left axis) and proximity to true MAE (right axis). }
    \label{fig:uniform_full}
\end{figure}
\subsection{Uniform Censorship}
The five subplots in Figure~\ref{fig:uniform_full} demonstrate the metrics performance for uniform censoring distributions. 
The last column in each subplot shows the ablation study of MAE-population-pseudo-observation (MAE-Pop-PO), in addition to the true MAE and six MAE-inspired metrics presented in Section~\ref{sec:mae_censor}.
All the MAE-inspired metrics discussed in Section~\ref{sec:mae_censor} can accurately identify the top-three models in all five datasets.
While MAE-margin is the closest one to the true MAE score in the GBM dataset, MAE-PO has the smallest difference in SUPPORT, METABRIC, MIMIC-A, and MIMIC-H in terms of true MAE proximity.
The results of MAE-Pop-PO show it neither has advantages in ranking the models (incorrectly ranking the top-three models for SUPPORT) nor can approximate the true MAE value (second largest in 4 subplots).  
We can conclude that for the uniform distribution, MAE-PO is the best here with MAE-margin as the second best one. 

\begin{figure}[!t]
    \centering
    \includegraphics[width=\columnwidth]{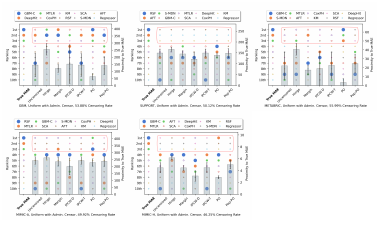}
    \caption{Evaluation metrics comparison on uniform with administrative censoring in terms of ranking accuracy (left axis) and proximity to true MAE (right axis). }
    \label{fig:uniform_admin_full}
\end{figure}
\subsection{Uniform with Administrative Censorship}
The five subplots in Figure~\ref{fig:uniform_admin_full} demonstrate the metrics performance for uniform censoring distributions with administrative censoring.
Due to the large percentage of administrative censoring, all semi-synthetic datasets have very large percentage censoring rates. 
The MAE-PO is again the best here, in both ranking performance (as it is the only one that correctly identifies the top-three models in GBM and METABRIC, the only one that identifies two of the top-three models for MIMIC-A, and the only one that identifies one of the top-three models for MIMIC-H) and closeness to the true MAE (significantly better in GBM, METABRIC and MIMIC-H with $p$-value $<$ 0.05, and one of the best in SUPPORT and MIMIC-A).

MAE-margin is the runner-up as it can identify parts of the best-performing models and has the second closest difference to true MAE. 
Between IPCW-D and IPCW-T, the performance does not have a significant difference in both ranking and proximity to true MAE.  
However, IPCW-D is associated with quite large error bars, which may be because the accuracy of later uncensored subjects will dominate the score (as we discussed in Section~\ref{sec:ipcw-d}).

Note 
that the GBM-C model, 
while doing very well (best in GBM and METABRIC) for the true MAE, does not achieve a good ranking for uncensored subjects (7th in GBM, 9th in METABRIC, and 10th in MIMIC-H for MAE-uncensored score).
Most other MAE variants also consider GBM-C as an ``inefficient'' model, while only the pseudo-observation can disclose GBM-C's true performance (1st in GBM and METABRIC, and 4th in MIMIC-H).

As to the results of MAE-Pop-PO, it has relatively large errors to the true MAE, and it shows no benefit for identifying the top models.

\begin{figure}[!t]
    \centering
    \includegraphics[width=\columnwidth]{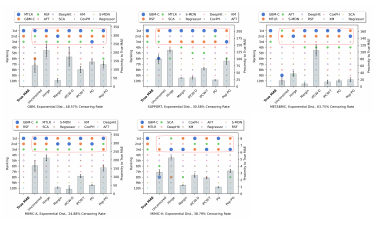}
    \caption{Evaluation metrics comparison on exponential censoring in terms of ranking accuracy (left axis) and proximity to true MAE (right axis). }
    \label{fig:exp_full}
\end{figure}
\subsection{Exponential Censorship}
The five subplots in Figure~\ref{fig:exp_full} demonstrate the metrics performance for exponential censoring distributions.
MAE-margin and MAE-PO show comparable performance in these five semi-synthetic datasets. 
MAE-margin correctly identifies the top models in GBM and MIMIC-A datasets and has a significantly smaller difference to true MAE compared to MAE-PO ($p$-value $<0.05$) on GBM, METABRIC, and MIMIC-A datasets.
MAE-PO, on the other side, excels in identifying all the top-three models in all five semi-synthetic datasets, and has a significantly smaller difference to true MAE ($p$-value $<0.05$) on  SUPPORT and MIMIC-H datasets.

We notice that MAE-IPCW-D is always associated with a larger standard deviation when it comes to the proximity to true MAE, this is again due to the reason we discussed in Section~\ref{sec:ipcw-d}.
MAE-Pop-PO does not provide any advantages when ranking models, nor can it approximate the true MAE value.

\begin{figure}[ht]
    \centering
    \includegraphics[width=\columnwidth]{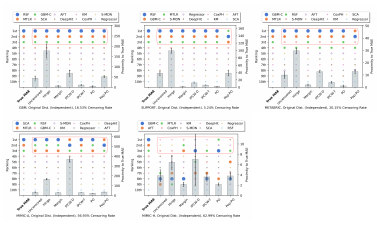}
    \caption{Evaluation metrics comparison on feature-independent original censoring in terms of ranking accuracy (left axis) and proximity to true MAE (right axis). }
    \label{fig:original_ind_full}
\end{figure}
\subsection{Feature-Independent Original Censorship}
The five subplots in Figure~\ref{fig:original_ind_full} demonstrate the metrics performance for feature-independent original censoring distribution.
Among all the evaluation metrics, margin and pseudo-observation perform equally the best for identifying the top three performing models (correctly identifying the best models for GBM, SUPPORT, METABRIC, and MIMIC-A, while recognizing two of the top-three model for MIMIC-H). 
MAE-PO has a slight advantage in the proximity to true MAE. Its value is closer to the true MAE on GBM and significantly closer on METABRIC and MIMIC-A datasets. 
In addition, we also observe that IPCW-D is always associated with a large variance (reason explained in Section~\ref{sec:ipcw-d}). 
As to the results of MAE-Pop-PO, it again shows no promising results compared to MAE-margin or MAE-PO.

\begin{figure}[ht]
    \centering
    \includegraphics[width=\columnwidth]{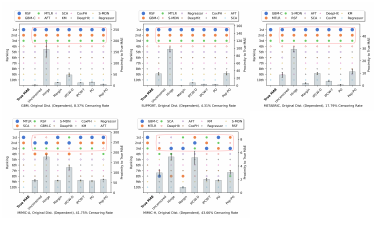}
    \caption{Evaluation metrics comparison on feature-dependent original censoring in terms of ranking accuracy (left axis) and proximity to true MAE (right axis). }
    \label{fig:original_dep_full}
\end{figure}
\subsection{Feature-Dependent Original Censorship}
The five subplots in Figure~\ref{fig:original_dep_full} demonstrate the metrics performance for feature-independent original censoring distribution.
MAE-uncensored 
performs the best on GBM datasets,
which may 
be due to the low synthetic censoring rate of this dataset (8.37\% censoring rate), 
meaning 
the whole dataset could be approximately represented by the uncensored population.
Among the MAE metrics that can handle the censored subjects, MAE-margin, IPCW-T, and MAE-PO perform equally well on GBM and MIMIC-A datasets.
For the METABRIC dataset, pseudo-observation is the best 
metric as it has the significantly lowest error to true MAE among all the metrics that can identify the top-three performing models.  
Pseudo-observation is also the optimal metric for the MIMIC-H dataset, as it is the only metric that can identify the top three models (GBM-C, MTLR, and SCA). 

\begin{figure}[ht]
    \centering
    \vspace*{-0.3in}
    \includegraphics[width=0.67\columnwidth]{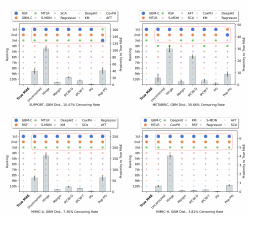}
    \caption{Evaluation metrics comparison on censoring distribution from GBM dataset in terms of ranking accuracy (left axis) and proximity to true MAE (right axis). }
    \label{fig:gbm_full}
\end{figure}
\subsection{External GBM Dataset Censorship}
The four subplots in Figure~\ref{fig:gbm_full} demonstrate the metrics performance for external dataset censor distribution using the GBM dataset.
We only have four subplots because the GBM dataset with external dataset censoring will just be the same as feature-independent original censorship.
MAE-PO is the optimal metric in all four semi-synthetic datasets, as it correctly identifies the top-three models in all cases, and has the significantly lowest error to true MAE among all other MAE-inspired metrics.

\end{document}